\documentclass{article} 
\usepackage{iclr_2027_submission/iclr2027_conference,times}

\usepackage{amsmath,amsfonts,bm}
\usepackage{amsthm}
\usepackage{amssymb}

\def\eqref#1{equation~\ref{#1}}

\def\1{\bm{1}}

\def\vc{{\bm{c}}}

\def\vm{{\bm{m}}}

\def\vs{{\bm{s}}}

\def\vv{{\bm{v}}}

\def\vx{{\bm{x}}}

\DeclareMathAlphabet{\mathsfit}{\encodingdefault}{\sfdefault}{m}{sl}
\SetMathAlphabet{\mathsfit}{bold}{\encodingdefault}{\sfdefault}{bx}{n}

\newcommand{\rbr}[1]{\left(#1\right)}

\newcommand{\nbr}[1]{\left\|#1\right\|}

\newtheorem{lemma}{Lemma}

\newtheorem{assumption}{Assumption}
\newtheorem{theorem}{Theorem}

\newcommand{\W}{W_2}
\newcommand{\law}{\mathrm{law}}

\usepackage[dvipsnames]{xcolor}
\usepackage{fontawesome5}
\usepackage[
    colorlinks=true,
    linkcolor=red,
    citecolor=RoyalBlue,
]{hyperref}

\usepackage{url}

\usepackage{algorithm}
\usepackage{algorithmic}
\usepackage{wrapfig}
\usepackage{enumitem}
\usepackage{booktabs}
\usepackage{multirow}
\usepackage{graphicx}
\usepackage{comment}
\usepackage[table]{xcolor}
\usepackage{caption}
\usepackage{makecell} 
\usepackage{color-edits}

\title{\ours: Bridging Train-Inference Gap for Flow-Based GRPO with Langevin Correction}

\author{
  Yingqing Guo
  \And Hui Yuan
  \And Zijian He
  \And Mengdi Wang
  \And Zheng Ding
  \AND
  \normalfont\href{https://github.com/yingqing001/LC_GRPO}{\faGithub\ \texttt{github.com/yingqing001/LC\_GRPO}}
}

\usepackage{xspace}
\newcommand{\ours}{LC-GRPO\xspace}

\begin{document}

\maketitle

\begin{abstract}

Flow-based generative models are typically sampled by solving a deterministic ordinary differential equation (ODE), whereas reinforcement learning (RL) requires stochastic rollouts for exploration and optimization. Existing online RL methods for flow models therefore replace the inference-time ODE with a stochastic differential equation (SDE) during training. Although the ODE and SDE share the same marginal distribution in continuous time, they can behave differently after finite-step discretization. In particular, SDE rollouts often become blurry as exploration noise increases, creating a mismatch between training rollouts and inference-time ODE samples. We introduce \textbf{\ours}, a flow-based GRPO framework with Langevin correction. Each rollout transition takes an inference-aligned ODE Euler step followed by a stochastic Langevin correction targeting the marginal at the resulting timestep. The required score is recovered directly from the flow velocity, requiring no additional score model, while the transition remains an isotropic Gaussian with tractable likelihood. We theoretically show that, under suitable conditions, a single Langevin correction step reduces the error of an imperfect ODE Euler step and, under matched randomness, yields a more accurate transition than standard Euler--Maruyama discretization of the reverse SDE. Additionally, experiments on SD3.5-Medium, FLUX.1-Dev, and HunyuanVideo show that \ours consistently improves reward optimization across text-to-image and text-to-video tasks, preserves generation quality, and substantially narrows the gap between stochastic training rollouts and deterministic ODE inference.

{\let\thefootnote\relax\footnotetext{Correspondence to: \texttt{yg6736@princeton.edu}, \texttt{zhengding.cs@gmail.com}.

}}

\end{abstract}

\section{Introduction}

Flow matching and rectified-flow models have emerged as a powerful foundation for visual generation, supporting high-quality text-to-image and text-to-video synthesis through deterministic ordinary differential equation (ODE) sampling \citep{liu2022flow,lipman2022flow,esser2024scaling,kong2024hunyuanvideo}.
Compared with stochastic diffusion sampling, the near-straight transport paths learned by these models enable accurate generation with relatively few sampling steps, making them especially attractive for large-scale visual models.
As generative models have become increasingly capable, post-training with external feedback has emerged as a particularly effective paradigm for aligning model behavior with complex objectives that are difficult to capture through supervised learning alone. It has shown strong success across diverse, often non-differentiable objectives, including human preferences, visual quality, compositional correctness, and accurate text rendering \citep{xu2026qwen,gao2025seedream,gao2025seedance,wu2025hunyuanvideo}. As a common approach in this alignment setting, online reinforcement learning (RL) methods such as DDPO \citep{black2024training} and, more recently, Group Relative Policy Optimization (GRPO) \citep{liu2026flow,xue2025dancegrpo}, directly optimize reward signals on model-generated samples.

However, one important issue emerges: \textit{the mismatch between model-generated samples during training and inference}, which can make the optimization signal used during training misaligned with inference-time behavior. This issue is not unique to visual generation: a similar training-inference gap has also been observed in LLM RL post-training, where separate inference and training stacks can introduce numerical and algorithmic inconsistencies, causing the rollout policy used for data collection to differ from the policy being optimized \citep{qi2025defeating,wasti2025no,qiu2026fp8,yao2025your,liang2026mirage}.

\begin{figure}[t]
\centering	\includegraphics[width=0.9\linewidth]{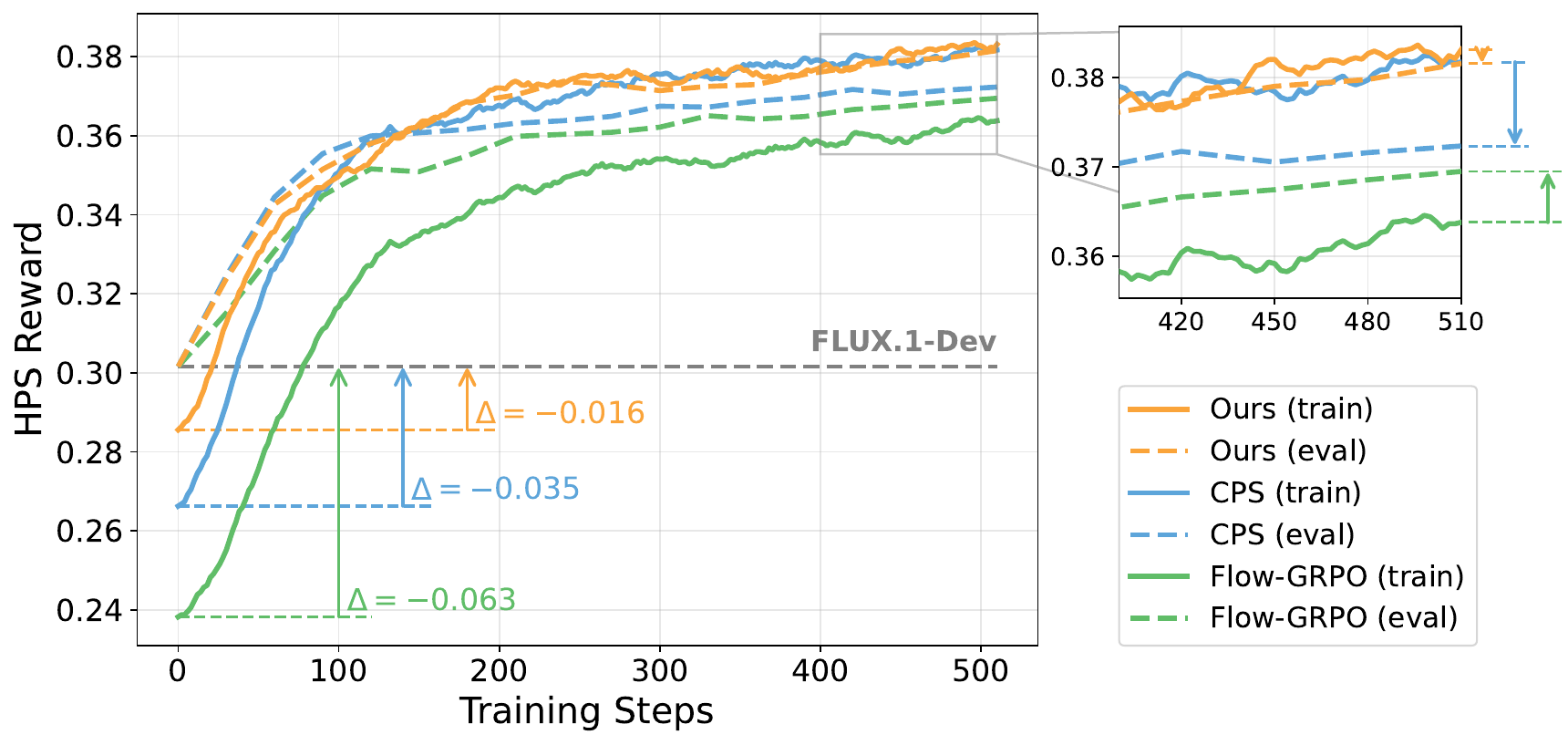} 
\includegraphics[width=0.9\linewidth]{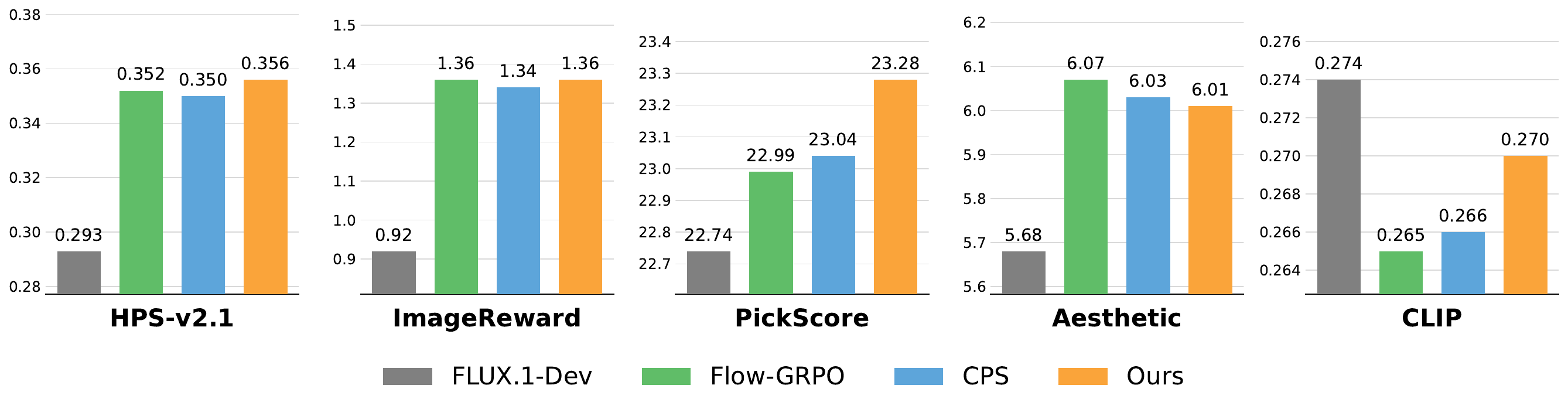}
	\caption{Performance of \ours. \textbf{Top:} Langevin correction narrows the gap between stochastic RL rollouts and ODE inference at test time, as shown at step 0, enabling \ours to achieve the smallest training--evaluation gap. \textbf{Bottom:} Evaluation results on DrawBench~\citep{saharia2022photorealistic}. }
  \label{fig:close_gap}
  \vspace{-10pt}
\end{figure}

The training–inference mismatch is even more fundamental in online RL for flow models. At test time, flow models are typically sampled by solving a deterministic reverse-time ODE. Reinforcement learning, however, requires \textit{stochasticity for exploration}, allowing the policy to sample multiple actions from the same state. Existing flow-based GRPO methods therefore replace the ODE with a theoretically equivalent stochastic differential equation (SDE) for rollout generation \citep{liu2026flow}.
In continuous time, the ODE and SDE share the same marginal distributions. However, in practice, sampling is typically performed with only a small number of discretization steps. The stochastic Brownian increments in SDE rollouts can introduce additional finite-step error, causing their behavior to deviate from deterministic ODE sampling \citep{kloeden2012numerical}. Empirically, we observe that applying SDE sampling to flow models produces noticeably blurrier and noisier samples than the  ODE sampling \citep{liu2022flow,wang2025coefficients}.

Several recent works have attempted to reduce this training--inference mismatch in flow-based online RL. MixGRPO \citep{li2025mixgrpo} combines SDE and ODE sampling during rollout generation, but leaves the underlying SDE dynamics unchanged; its benefit is therefore largely orthogonal to the fundamental SDE--ODE mismatch. Coefficients-preserving sampling (CPS) \citep{wang2025coefficients} instead adopts a DDIM-style stochastic sampler, but its empirical gains remain limited.

We address this limitation with \textbf{\ours}, a flow-based GRPO framework based on Langevin correction (Alg.~\ref{alg:sampling_lc}). Each rollout transition first takes the same ODE Euler step used by the inference-time sampler, followed by a single Langevin correction step at the resulting timestamp. This predictor--corrector construction separates transport from exploration: the ODE step preserves the inference-time dynamics, whereas the Langevin step injects stochasticity while preserving the target marginal distribution. Importantly, the required score is recovered directly from the flow velocity through Tweedie’s formula, requiring no additional score model, while the resulting transition remains an isotropic Gaussian with a tractable density for policy optimization.

The key motivation is that Langevin correction provides a more accurate rollout update than directly simulating the reverse SDE, thereby reducing the gap to the inference-time ODE. Intuitively, the Langevin step acts as a correction that mitigates local sampling error while retaining the stochasticity needed for exploration. We support this intuition both theoretically and empirically. Our analysis shows that, under matched randomness and the same exploration strength, Langevin-corrected sampling incurs lower error than the standard reverse-SDE Euler--Maruyama update (Theorem~\ref{thm:sde informal}). Consistently, Fig.~\ref{fig:close_gap} shows that \ours already has the smallest gap to evaluation-time performance at training step 0. By the end of training, it maintains the smallest train--evaluation gap and achieves the best evaluation performance among the GRPO baselines.

We validate \ours on image and video generation using SD3.5-Medium \citep{esser2024scaling}, FLUX.1-Dev \citep{flux2024}, and HunyuanVideo \citep{kong2024hunyuanvideo}, covering visual text rendering, human-preference alignment, multi-reward optimization, and video-quality optimization (Sec.~\ref{sec:experiment}). Across these settings, \ours consistently outperforms Flow-GRPO, DanceGRPO, and coefficients-preserving sampling on the optimized rewards, while preserving auxiliary image and video quality metrics under rollout-computation budgets that match or favor the baselines.

To summarize, our main contributions are as follows:
\begin{enumerate}[itemsep=1pt, parsep=1.5pt, topsep=1pt, partopsep=0pt, leftmargin=20pt]
    \item We identify finite-step SDE discretization as a key source of the training--inference gap in flow-based GRPO and characterize its impact on rollout quality and reward alignment (Figs.~\ref{fig:close_gap},~\ref{fig:rollout imgs}).
    \item We introduce \ours, which combines inference-aligned ODE transport with a distribution-preserving Langevin correction to enable accurate, stochastic, likelihood-tractable RL rollouts.
    
    \item We establish theoretical guarantees for the proposed correction (Theorems~\ref{thm:imp_ode_informal},~\ref{thm:sde informal}) and demonstrate consistent empirical improvements across text-to-image and text-to-video tasks (Sec.~\ref{sec:experiment}).
\end{enumerate}

\section{Related Work}
We first review alignment methods for diffusion and flow models, including
post-training with external feedback and inference-time guidance, and then
focus on closing the training--inference gap in flow-based GRPO and LLM RL
post-training.

\paragraph{RL for Diffusion and Flow Models.}
Aligning diffusion and flow models during post-training generally follows one
of three strategies.
(1) Direct reward backpropagation, which fine-tunes the model by
backpropagating through a differentiable reward function
\citep{clark2024directly,prabhudesai2023aligning}.
(2) Reward-augmented regression losses, which combine external feedback with
a diffusion regression objective. These include Reward-Weighted Regression
(RWR) \citep{lee2023aligning,dong2023raft}; Direct Preference Optimization
(DPO), which leverages paired preference data with binary human-feedback
rewards
\citep{rafailov2023direct,wallace2024diffusion,yang2024using};
and DiffusionNFT \citep{zheng2025diffusionnft}, which generalizes DPO by
incorporating negative samples with dense, continuous rewards.
(3) PPO-style policy-gradient methods, ranging from early approaches such as
DPOK \citep{fan2023dpok} and DDPO \citep{black2024training} to recent GRPO
extensions for diffusion and flow models
\citep{liu2026flow,xue2025dancegrpo,deng2026densegrpo,
savani2026stepwise,ding2025treegrpo}.
Beyond model post-training, complementary inference-time alignment methods
guide diffusion and flow sampling using time-dependent classifiers or
gradients derived from downstream objectives
\citep{song2020score,dhariwal2021diffusion,chung2022diffusion,yuan2023reward,guo2024gradient,guo2026training}.

\paragraph{Training--Inference Gap in Flow-Based GRPO.}
Recent work explores reducing the mismatch between stochastic training
rollouts and deterministic inference samplers in flow-based GRPO.
MixGRPO \citep{li2025mixgrpo} integrates SDE and ODE sampling within the
rollout to improve training efficiency.
CPS \citep{wang2025coefficients} instead adopts DDIM sampling to improve
sample quality.
SAGE-GRPO \citep{zheng2026manifold} proposes a concise formulation of the
noise standard deviation that better aligns the sampling process with the
data manifold for video generation.

\paragraph{Training--Inference Mismatch in LLM RL Post-Training.}
Practical LLM RL systems often use separate inference engines for rollout
generation and separate frameworks for training, which can introduce a
mismatch.
One line of work addresses numerical inconsistencies
\citep{qi2025defeating,qiu2026fp8,wasti2025no}; for example,
\citet{qi2025defeating} switch from BF16 to FP16 during RL fine-tuning.
A complementary line of work takes an algorithmic approach by incorporating
sampler-side information into the training update
\citep{yao2025your,liu2025speed,liang2026mirage}.
Notably, \citet{yao2025your} use the training-to-sampler probability ratio as
a clipped correction weight.
\section{Preliminaries}

\paragraph{Flow Matching.} 
Let $\boldsymbol{x}_0 \sim X_0$ denote a sample from the data distribution and
$\boldsymbol{x}_1 \sim X_1$ denote a noise sample from the prior distribution, such as a standard Gaussian. Rectified Flow \citep{liu2022flow,lipman2022flow} constructs a time-dependent interpolation between $\boldsymbol{x}_0$ and $\boldsymbol{x}_1$:
$$
    \boldsymbol{x}_t = (1-t)\boldsymbol{x}_0 + t\boldsymbol{x}_1,
    \quad t \in [0,1].
$$
Under this convention, $t=0$ corresponds to data and $t=1$ to noise. Generation is therefore performed by solving the reverse-time ODE from the prior distribution back to the data distribution: 
\begin{equation}\label{eq:inf_ode}
    \mathrm{d} \boldsymbol{x}_t = \boldsymbol{v}_t  \mathrm{d} t, 
\end{equation}
with $\boldsymbol{x}_1 \sim X_1$ and $t:1 \to 0$. The velocity field $\boldsymbol{v}_t$ transports samples from the prior distribution to the data distribution. In practice, it is approximated by a transformer $\boldsymbol{v}_\theta(\boldsymbol{x}_t,t)$, trained with the flow matching regression objective:
$
    \min_{\theta}
    \mathbb{E}_{t,\boldsymbol{x}_0,\boldsymbol{x}_1}
    \left[
        \left\|
        \boldsymbol{v}_\theta(\boldsymbol{x}_t,t)
        -
        (\boldsymbol{x}_1-\boldsymbol{x}_0)
        \right\|_2^2
    \right].
$
The learned velocity field is then used to solve the reverse-time ODE (Eq.~\ref{eq:inf_ode}) for generation.

\paragraph{SDE Sampling as RL Rollout.}
The inference process of flow matching can be formulated as a Markov decision process (MDP)~\citep{black2024training}. At denoising step $t$, the state is defined as $\boldsymbol{s}_t:= (\boldsymbol{c}, t, \boldsymbol{x}_t)$, the action is the next denoised sample $\boldsymbol{a}_t:= \boldsymbol{x}_{t-1}$ predicted by the model, and the policy is given by $\pi(\boldsymbol{a}_t \mid \boldsymbol{s}_t):=  p_\theta(\boldsymbol{x}_{t-1} \mid \boldsymbol{x}_t, \boldsymbol{c})$.

However, ODE-based sampling is deterministic and induces no action-level stochasticity, which blocks the stochastic exploration required by reinforcement learning. To resolve this mismatch, Flow-GRPO~\citep{liu2026flow} and DanceGRPO~\citep{xue2025dancegrpo} replace the reverse-time ODE with an equivalent SDE:
\begin{equation}\label{eq:inf_sde}
\mathrm{d}\boldsymbol{x}_t = \left[ \boldsymbol{v}_t(\boldsymbol{x}_t,t) + \frac{\sigma_t^2}{2t} \left( \boldsymbol{x}_t + (1-t)\boldsymbol{v}_t(\boldsymbol{x}_t,t) \right) \right]\mathrm{d}t + \sigma_t \mathrm{d}\boldsymbol{w}_t ,
\end{equation}
which adds noise at each denoising step, enabling trajectories to serve as stochastic RL rollouts.

\paragraph{GRPO on Flow Models.} Given a prompt \(\vc\), the flow model \(p_{\theta}\) generates \(G\) independent denoising trajectories \(\{(\vx^i_t)_{t=0}^T\}_{i=1}^G\), each ending at sample \(\vx^i_0\). The advantage of the \(i\)-th sample is computed by normalizing its final reward \(R(\vx^i_0,\vc)\) against the rewards of other samples in the same group:
\begin{equation}
\label{eq:grpo_advantage}
\hat{A}^i_t = \frac{R(\vx^i_0, \vc) - \text{mean}(\{R(\vx^j_0, \vc)\}_{j=1}^G)}{\text{std}(\{R(\vx^j_0, \vc)\}_{j=1}^G)}.
\end{equation}
The policy is trained by maximizing the GRPO objective:
\begin{equation}
\mathcal{J}(\theta) = \mathbb{E}_{\vc\sim \mathcal{C}, \{\vx^i\}_{i=1}^G\sim \pi_{\theta_\text{old}}(\cdot\mid \vc)} f(r,\hat{A},\theta, \varepsilon, \beta),
\label{eq:grpoloss}
\end{equation}
where
\begin{align*}
f(r,\hat{A},\theta, \varepsilon, \beta) &=
\frac{1}{G}\sum_{i=1}^{G} \frac{1}{T}\sum_{t=1}^{T} \Bigg( 
\min \Big( r^i_t(\theta) \hat{A}^i_t,  
\ \text{clip} \Big( r^i_t(\theta), 1 - \varepsilon, 1 + \varepsilon \Big) \hat{A}^i_t \Big)
- \beta D_{\text{KL}}(\pi_{\theta} || \pi_{\text{ref}}) 
\Bigg), \\
r^i_t(\theta) &=
\frac{p_{\theta}(\vx^i_{t-1} \mid \vx^i_t, \vc)}{p_{\theta_{\text{old}}}(\vx^i_{t-1} \mid \vx^i_t, \vc)}.
\end{align*}
In flow models, an action is one reverse-time sampling step from \(\vx^i_t\) to \(\vx^i_{t-1}\), obtained by discretizing Eq.~\ref{eq:inf_sde} (see Eq.~\ref{eq:discretize_sde}). The resulting transition probability $p_{\theta}(\vx^i_{t-1} \mid \vx^i_t, \vc)$ is an isotropic Gaussian, whose density can be computed in closed form.

\section{Method}\label{sec:method}
In this section, we first revisit the training--inference gap in standard GRPO for flow models, where SDE sampling is used to generate RL rollouts. We then propose \ours, which introduces an additional Langevin correction step to reduce the gap between stochastic RL rollouts and deterministic ODE inference. We validate this approach both theoretically and empirically.
\paragraph{Training--Inference Gap from SDE Rollouts.}
During inference, although the ODE (Eq.~\ref{eq:inf_ode}) and SDE (Eq.~\ref{eq:inf_sde}) reach the same marginal distribution at each timestep in continuous time \citep{kloeden2012numerical}, their finite-step numerical samplers can differ. ODE Euler step is usually taken for current advanced image and video generation \citep{esser2024scaling,wu2025qwen,wan2025wan}:
\begin{equation}\label{eq:discretize_ode}
    \boldsymbol{x}_{t - \Delta t} = \boldsymbol{x}_{t} - \boldsymbol{v}_t(\boldsymbol{x}_t,t) \Delta t.
\end{equation}
Euler--Maruyama discretization is commonly used for SDE rollouts in GRPO for flow models:
\begin{equation}\label{eq:discretize_sde}
    \boldsymbol{x}_{t - \Delta t} = \boldsymbol{x}_{t} - \left[ \boldsymbol{v}_t(\boldsymbol{x}_t,t) + \frac{\sigma_t^2}{2t} \left( \boldsymbol{x}_t + (1-t)\boldsymbol{v}_t(\boldsymbol{x}_t,t) \right) \right] \Delta t  + \sigma_t \sqrt{\Delta t} \, \xi,
\end{equation}
where $\xi \sim \mathcal{N}(0, I_d)$. For the diffusion coefficient, Flow-GRPO sets $\sigma_t = \eta \sqrt{t/(1-t)}$. Here, $\eta$ controls the noise level, and hence the amount of stochasticity and exploration during RL training. Both methods show that larger $\eta$ values often accelerate GRPO training and lead to higher rewards.

As shown in Fig.~\ref{fig:rollout imgs}, SDE-based sampling tends to produce lower-quality samples than ODE-based sampling, yielding blurrier images; this degradation becomes more severe as noise level increases. Similar behavior has been observed in prior work~\citep{lipman2022flow,song2020score}. This discrepancy creates a training--inference gap in RL for flow models: SDE rollouts are worse than inference-time ODE samples, making their rewards a less accurate and less effective learning signal.

\begin{figure}[t]
\centering	\includegraphics[width=0.9\linewidth]{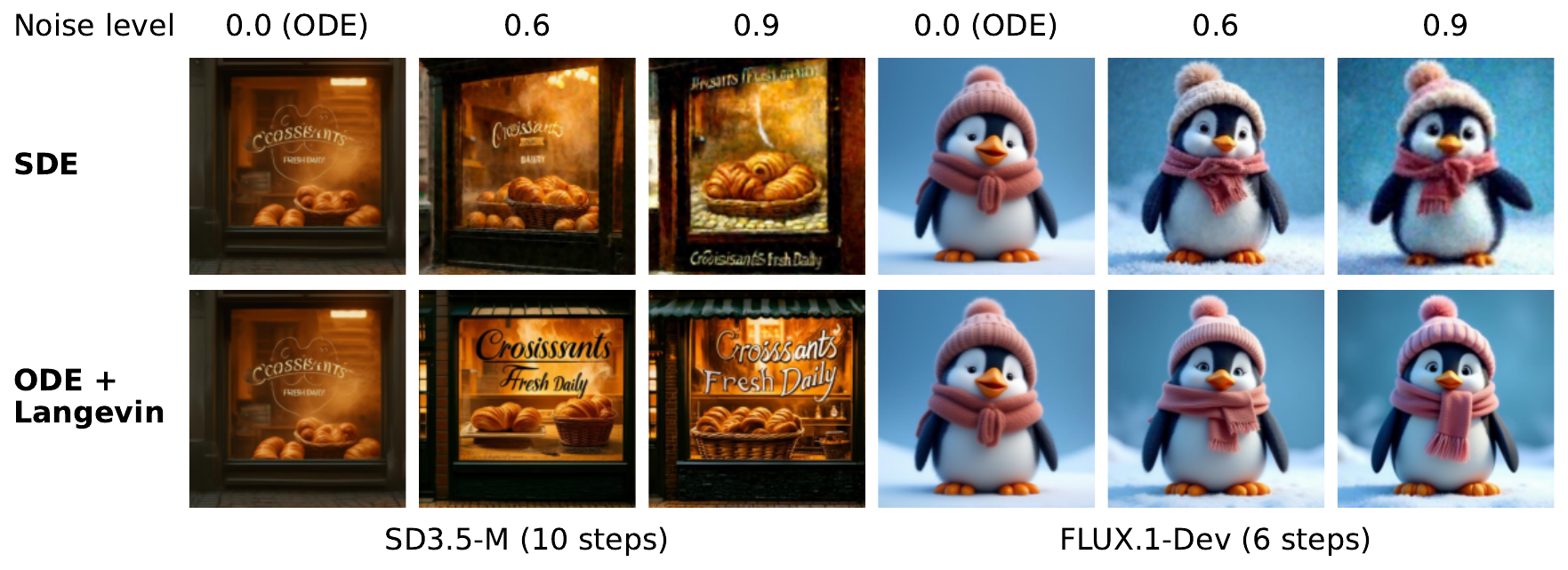}\\
	\caption{\textbf{Sampling Method Comparison}. As the noise level\protect\footnotemark  increases, SDE-based sampling degrades sample quality, whereas ODE sampling with Langevin correction preserves sample clarity.
  }\label{fig:rollout imgs}
\end{figure}

\paragraph{Langevin Step: Distribution-Preserving Stochastic Correction.}
Motivated by the training--inference gap, we aim to design a sampler that narrows this mismatch. SDE-based sampling accumulates error from both drift discretization and Brownian increments. To avoid introducing additional error through direct SDE simulation, we first take an ODE Euler step from $\boldsymbol{x}_t$ to obtain an intermediate proposal $\boldsymbol{x}'$ at time $t-\Delta t$. We then seek a stochastic update that satisfies two requirements: (i) it injects stochasticity while preserving the target distribution at time $t-\Delta t$; (ii) it improves sample quality by mitigating local numerical and discretization errors; (iii) it admits an analytically tractable transition likelihood for policy optimization in RL.

We adopt a simple yet effective strategy: after each ODE step, we apply one Langevin dynamics step initialized at $\boldsymbol{x}'$, with sampling $\boldsymbol{z}\sim\mathcal{N}(\boldsymbol{0},I_d)$, using the score estimate at time $t-\Delta t$ :
\begin{equation}\label{eq:lc_step}
    \boldsymbol{x}_{t - \Delta t} = \boldsymbol{x}^\prime + \epsilon_t \, \boldsymbol{s}(\boldsymbol{x}^\prime,t-\Delta t) + \sqrt{2\epsilon_t}\, \boldsymbol{z},
\end{equation}
where $\epsilon_t$ is the step size. The score function is obtained via Tweedie's formula \citep{efron2011tweedie} as $\boldsymbol{s}(\boldsymbol{x},t) = - \rbr{\boldsymbol{x} + (1-t)\boldsymbol{v}(\boldsymbol{x},t)}/t $. Following \citet{song2020score}, we set $\epsilon_t \propto \left(\nbr{\boldsymbol{z}} / \nbr{\boldsymbol{s}}\right)^2$. For stability, we use its expectation, which gives $\sqrt{2\epsilon_t} = \eta (t-\Delta t)$, where $\eta$ controls the noise level.

This correction has the desired properties. First, Langevin dynamics driven by the score at time $t-\Delta t$ leaves the marginal $p_{t-\Delta t}$ invariant, thereby introducing randomness while preserving that marginal. Second, it serves as a corrector step: prior work on score-based generative modeling has shown that such stochastic correction can improve sampling quality~\citep{song2020score}. As Fig~\ref{fig:rollout imgs} shows, sampling with Langevin correction preserves sample quality even under large noise levels.

\footnotetext{The noise level refers to $\eta$ in the SDE (Eq.~\ref{eq:discretize_sde}), where $\sigma_t = \eta \sqrt{t/(1-t)}$ following Flow-GRPO, and in the Langevin step (Eq.~\ref{eq:lc_step}), where we set $\sqrt{2\epsilon_t} = \eta (t-\Delta t)$. For the same $\eta$, the sum of the variance of the Gaussian noise over all timesteps is larger for Langevin than for the SDE.}

\renewcommand{\algorithmiccomment}[1]{%
  \hfill $\triangleright$\,\textit{#1}%
}

\begin{algorithm}
\caption{Langevin Correction Sampling}
\label{alg:sampling_lc}
\begin{algorithmic}[1]
\REQUIRE flow model $\boldsymbol{v}_{\theta}$, Langevin step size $\epsilon_t$, time schedule $\{t_i\}_{i=1}^{N}$ with $t_1=1,\; t_N=0$
\STATE  $\boldsymbol{x}_1 \sim \mathcal{N}(\boldsymbol{0}, I_d)$
\FOR{$i=1,\ldots,N-1$}
    \STATE
    $\boldsymbol{x}' \leftarrow
    \boldsymbol{x}_{i}
    -(t_i-t_{i+1})
    \boldsymbol{v}_{\theta}(\boldsymbol{x}_{i},t_i)$
    \COMMENT{Euler step} 

    \IF{$i < N - 1$}
        \STATE
        $\boldsymbol{s}_{t_{i+1}}(\boldsymbol{x}')
        \leftarrow
        -\dfrac{
        \boldsymbol{x}'
        +(1-t_{i+1})
        \boldsymbol{v}_{\theta}(\boldsymbol{x}',t_{i+1})
        }{t_{i+1}}$
        \COMMENT{Score estimation}
        \STATE
        $\boldsymbol{z}\sim\mathcal{N}(\boldsymbol{0},I_d)$
        \STATE
        $\boldsymbol{x}_{i+1}\leftarrow
        \boldsymbol{x}'
        +\epsilon_{t_i}\boldsymbol{s}_{t_{i+1}}(\boldsymbol{x}')
        +\sqrt{2\epsilon_{t_i}}\,\boldsymbol{z}$
        \COMMENT{Langevin correction}
    \ELSE
        \STATE
        $\boldsymbol{x}_{i+1}\leftarrow\boldsymbol{x}'$
        \COMMENT{Final ODE step}
    \ENDIF
\ENDFOR
\RETURN $\boldsymbol{x}_N$
\end{algorithmic}
\end{algorithm}

As for the MDP we formulate for GRPO, the transition kernel $p_{\theta}(\vx_{t-\Delta t} \mid \vx_t, \vc)$ remains an isotropic Gaussian $\mathcal{N}(\vm_\theta(\vx_t), 2\epsilon_t I_d)$, whose mean $\vm_\theta(\vx_t) = \vx' + \epsilon_t \vs_\theta(\vx', t-\Delta t)$ is the deterministic part of the Langevin update in Alg.~\ref{alg:sampling_lc}, where $\vx' = \vx_t - \Delta t\,\vv_\theta(\vx_t, t)$ is the Euler step. Thus, the Langevin step size $\epsilon_t$ also controls the extent of exploration of the action policy.

The following theorem justifies the correction step in Alg.~\ref{alg:sampling_lc}: whenever the Euler step is imperfect, one Langevin step provably moves the sample closer to the target marginal. The full statement and proof are deferred to Appx.~\ref{appx:theory} and Appx.~\ref{proof:corrector-improves}, respectively.
\begin{theorem}[Informal]
Suppose the target $p_{t-\Delta t}$ is $\alpha$-strongly log-concave with
$L$-Lipschitz score. Let $\boldsymbol{x}_{\mathrm{ode}}$ be the result of one
backward Euler step to time $t-\Delta t$ (Eq.~\ref{eq:discretize_ode}), and let
$\varepsilon := \W\big(\mathrm{law}(\boldsymbol{x}_{\mathrm{ode}}),\, p_{t-\Delta t}\big) > 0$ be its error. If the Langevin step size satisfies $\epsilon_t \le \alpha/L^2$ and
$\epsilon_t \lesssim \alpha^2\varepsilon^2/(L^2 d)$, then one Langevin correction step
strictly reduces the error:
\begin{equation*}
    \W\big(\mathrm{law}(\boldsymbol{x}_{\mathrm{lc}}),\, p_{t-\Delta t}\big)
    \;<\;
    \W\big(\mathrm{law}(\boldsymbol{x}_{\mathrm{ode}}),\, p_{t-\Delta t}\big),
\end{equation*}
where $\boldsymbol{x}_{\mathrm{lc}}$ is the Langevin-corrected sample
(Eq.~\ref{eq:lc_step}) and $\W$ is the Wasserstein distance.
\label{thm:imp_ode_informal}
\end{theorem}

In flow-based GRPO, increasing the noise level, and thus the extent
of policy exploration, typically improves reward
optimization~\citep{liu2026flow} but degrades the sample quality of rollouts.
Preserving quality at a fixed randomness level is therefore
important. Below we show that when our Langevin-corrected sampler is
matched to the Euler--Maruyama SDE at the same randomness level, that
is, comparing Eq.~\ref{eq:discretize_sde} and Eq.~\ref{eq:lc_step}
and setting the Langevin step size so that
$\sqrt{2\epsilon_t} = \sigma_t \sqrt{\Delta t}$, i.e.,
$\epsilon_t = \sigma_t^2 \Delta t / 2$, Langevin correction sampling
is provably more accurate than the SDE. The formal statement and its proof are given in
Appx.~\ref{appx:theory} and Appx.~\ref{proof:main}, respectively.

\begin{theorem}[Informal]
Under mild regularity (Assumption~\ref{ass:reg}) and with the step size matched
to the noise level, $\epsilon_t = \sigma_t^2\Delta t/2$, both the corrector
$\boldsymbol{x}_{\mathrm{lc}}$ and the reverse-SDE step $\boldsymbol{x}_{\mathrm{sde}}$
(Eq.~\ref{eq:discretize_sde}) have leading-order Wasserstein error of order
$\Delta t^2$ (Eq.~\ref{eq:errors}). If the comparison condition Eq.~\ref{eq:lc_sde_cond_matched_noise} holds, the Langevin correction sample step is strictly more accurate:
\begin{equation*}
    \W\big(\mathrm{law}(\boldsymbol{x}_{\mathrm{lc}}),\, p_{t-\Delta t}\big)
    \;<\;
    \W\big(\mathrm{law}(\boldsymbol{x}_{\mathrm{sde}}),\, p_{t-\Delta t}\big),
\end{equation*}
for all sufficiently small $\Delta t$.
\label{thm:sde informal}
\end{theorem}

Beyond the above theoretical result, Fig.~\ref{fig:rollout imgs} provides empirical verification. Note that at the same noise level $\eta$, the sum of the variance of the Gaussian noise over all timesteps is actually larger for Langevin than for the SDE, implying at least as much exploration; nevertheless, our Langevin-corrected sampling produces substantially clearer samples than standard SDE-based sampling.

\section{Experiments}\label{sec:experiment}
In this section, we empirically evaluate \ours on text-to-image and text-to-video generation.

\subsection{Experimental Setup}

\paragraph{Image Generation.}
We experiment with two flow-based text-to-image models, SD3.5-Medium~\citep{esser2024scaling} and FLUX.1-Dev~\citep{flux2024}, both at $512\times512$ resolution. These models cover two guidance regimes: SD3.5 uses classifier-free guidance (CFG), while FLUX.1-Dev is guidance-distilled and uses no CFG at inference. We consider three reward-optimization settings:
\begin{enumerate}[itemsep=0pt, parsep=0pt, topsep=0pt, leftmargin=12pt]
\item \textit{Visual text rendering (OCR).} This is a verifiable reward that measures how accurately the model renders a target string. We follow Flow-GRPO~\citep{liu2026flow} for both the reward-assignment strategy and the training and evaluation datasets.
\item \textit{Human preference alignment.} We use HPS-v2.1~\citep{wu2023human} as the reward model, a preference model trained on human comparisons to predict which generated image humans prefer for a given prompt. We use the training and evaluation prompts from Pick-a-Pic~\citep{kirstain2023pick}.
\item \textit{ Multi-reward optimization.} To preserve text-image alignment while improving human preference, we jointly optimize CLIP score~\citep{hessel2021clipscore} and HPS-v2.1. We use the same training and evaluation prompts as in the single-reward HPS-v2.1 setting.
\end{enumerate}

\paragraph{Video Generation.}
For text-to-video generation, we experiment with HunyuanVideo~\citep{kong2024hunyuanvideo}. Following \citet{xue2025dancegrpo}, we use VideoAlign~\citep{liu2026improving} as the reward model and optimize its visual aesthetic quality component. We also report its motion quality and text--video alignment components for evaluation. As in DanceGRPO, training prompts are curated from VidProM~\citep{wang2024vidprom}; the evaluation dataset is from \citet{zheng2026manifold}.

\paragraph{Quality Evaluation Metrics.}
When optimizing a target reward, we monitor quality metrics to detect reward-induced degradation or reward hacking. For images, we report two quality metrics, Aesthetic Score and CLIP score, and three human-preference metrics, ImageReward~\citep{xu2023imagereward}, HPS-v2.1, and PickScore~\citep{kirstain2023pick}. All image metrics are computed on DrawBench~\citep{saharia2022photorealistic}. For videos, we report VBench~\citep{huang2023vbench} as quality metrics.

\begin{table*}[t]
    \centering
    \caption{\textbf{Evaluation Results on Image Generation.}
    Task metric columns denote the reward(s) optimized in each setting. \textbf{Bold} is the best result. \ours consistently outperforms all baselines. }
    \resizebox{0.9\linewidth}{!}{
        \begin{tabular}{lcccccccc}
            \toprule
            \multirow{2}{*}{\textbf{Model}}
            & \multicolumn{3}{c}{\textbf{Task Metric}}
            & \multicolumn{5}{c}{\textbf{Quality Metric}} \\
            \cmidrule(lr){2-4} \cmidrule(lr){5-9}
            & \textbf{OCR} & \textbf{HPS-v2.1} & \textbf{CLIP}
            & \textbf{Aesthetic} & \textbf{CLIP} & \textbf{ImgRwd} & \textbf{HPS-v2.1} & \textbf{PickScore} \\ 
            \midrule \rowcolor{gray!35}
            \multicolumn{9}{l}{\textit{\textbf{SD3.5-M}}} \\
            \midrule
            SD3.5-M & 0.569 & 0.295 & 0.288 & 5.38 & 0.283 & 0.83 & 0.279 & 22.35 \\
            \midrule 
            \multicolumn{9}{l}{\textit{(a) Visual Text Rendering: OCR}} \\
            \midrule
            Flow-GRPO & 0.914 & — & — & 5.31 & 0.288 & 0.94 & \textbf{0.281} & 22.43 \\
            CPS       & 0.935 & — & — & 5.20 & 0.287 & 0.75 & 0.265 & 22.16 \\
            \ours     & \textbf{0.960} & — & — & \textbf{5.33} & \textbf{0.291} & \textbf{1.00} & 0.280 & \textbf{22.46} \\
            \midrule
            \multicolumn{9}{l}{\textit{(b) Human Preference Alignment: HPS-v2.1}} \\
            \midrule
            Flow-GRPO & — & 0.381 & — & \textbf{6.30} & 0.267 & 1.41 & 0.357 & 22.89 \\
            CPS       & — & 0.371 & — & 6.12 & 0.270 & 1.38 & 0.344 & 22.62 \\
            \ours     & — & \textbf{0.393} & — & 6.11 & \textbf{0.280} & \textbf{1.44} & \textbf{0.367} & \textbf{23.03} \\
            \midrule
            \multicolumn{9}{l}{\textit{(c) Multi-Reward: HPS-v2.1 $+$ CLIP Score}} \\
            \midrule
            Flow-GRPO & — & 0.351 & 0.296 & \textbf{5.97} & 0.289 & 1.28 & 0.327 & 22.76 \\
            CPS       & — & 0.345 & 0.291 & 5.73 & 0.285 & 1.31 & 0.323 & 22.68 \\
            \ours     & — & \textbf{0.356} & \textbf{0.302} & 5.73 & \textbf{0.297} & \textbf{1.33} & \textbf{0.332} & \textbf{22.93} \\
            \midrule \rowcolor{gray!35}
            \multicolumn{9}{l}{\textit{\textbf{FLUX.1-Dev}}} 
            \\
            \midrule
            FLUX.1-Dev & — & 0.303 & 0.276 &5.68 & 0.274 & 0.92 & 0.293  & 22.74 \\
            \midrule
            \multicolumn{8}{l}{\textit{(a) Human Preference Alignment: HPS-v2.1}} \\
            \midrule
            Flow-GRPO & — & 0.378 & — & \textbf{6.07} & 0.265 &\textbf{1.36} & 0.352 & 22.99 \\
            CPS     & —  & 0.373 & — & 6.03 & 0.266 &1.34 & 0.350 & 23.04 \\
            \ours    & — & \textbf{0.384} & — & 6.01 & \textbf{0.270} & \textbf{1.36}& \textbf{0.356} & \textbf{23.28} \\
            \midrule
            \multicolumn{8}{l}{\textit{(b) Multi-Reward: HPS-v2.1 $+$ CLIP Score}} \\
            \midrule
            Flow-GRPO & — & 0.371 & 0.279 &\textbf{5.94} &0.273 & 1.37&0.347 & 23.18 \\
            CPS    & —   & 0.370 & 0.284 &5.90 &0.281 &\textbf{1.44} &0.350 & 23.28 \\
            \ours   & —  & \textbf{0.376} & \textbf{0.285}& 5.88 &\textbf{0.284}  &1.43 &\textbf{0.352} & \textbf{23.40} \\
            \bottomrule
        \end{tabular}
    }
    \label{tab:img_res}
    \vspace{-3mm}
\end{table*}

\paragraph{Baselines.}
We compare \ours against online reinforcement learning baselines based on GRPO. These methods share the same GRPO objective and differ primarily in how rollouts are sampled. The first baseline uses the standard SDE sampler in Eq.~\ref{eq:discretize_sde}: Flow-GRPO~\citep{liu2026flow} for text-to-image generation and DanceGRPO~\citep{xue2025dancegrpo} for text-to-video generation. The second baseline, CPS~\citep{wang2025coefficients}, instead samples rollouts with DDIM~\citep{song2020denoising}.

\paragraph{Fair Comparison on Computation.}
Since the compared methods mainly differ in rollout sampling, we match or favor the baselines in rollout computation. Each Langevin-corrected sampler step requires two NFEs; therefore, for each baseline, we run both the same number of rollout steps and twice that number, and report the better result. For SD3.5-M, FLUX.1-Dev, and HunyuanVideo, \ours uses 10, 6, and 8 rollout steps during training, while the baselines use 10/20, 6/12, and 16 steps, respectively. For evaluation, all methods use the same sampler: 40-step ODE Euler for SD3.5-M, 28 steps for FLUX.1-Dev, and 50 steps for HunyuanVideo. Thus, the proposed method does not use more computation than the baselines. Other hyperparameters are kept nearly identical, with details in Appendix~\ref{app:exp_setup}.

\subsection{Main Results}

As shown in Tables~\ref{tab:img_res} and~\ref{tab:video_res}, \ours consistently outperforms the baselines on both image and video generation tasks while preserving sample quality across auxiliary metrics. Qualitative results are shown in Fig.~\ref{fig:qualitative imgs}, with additional examples provided in Appendix~\ref{app:additional_res}.

\begin{figure}[htbp]
\centering	\includegraphics[width=0.95\linewidth]{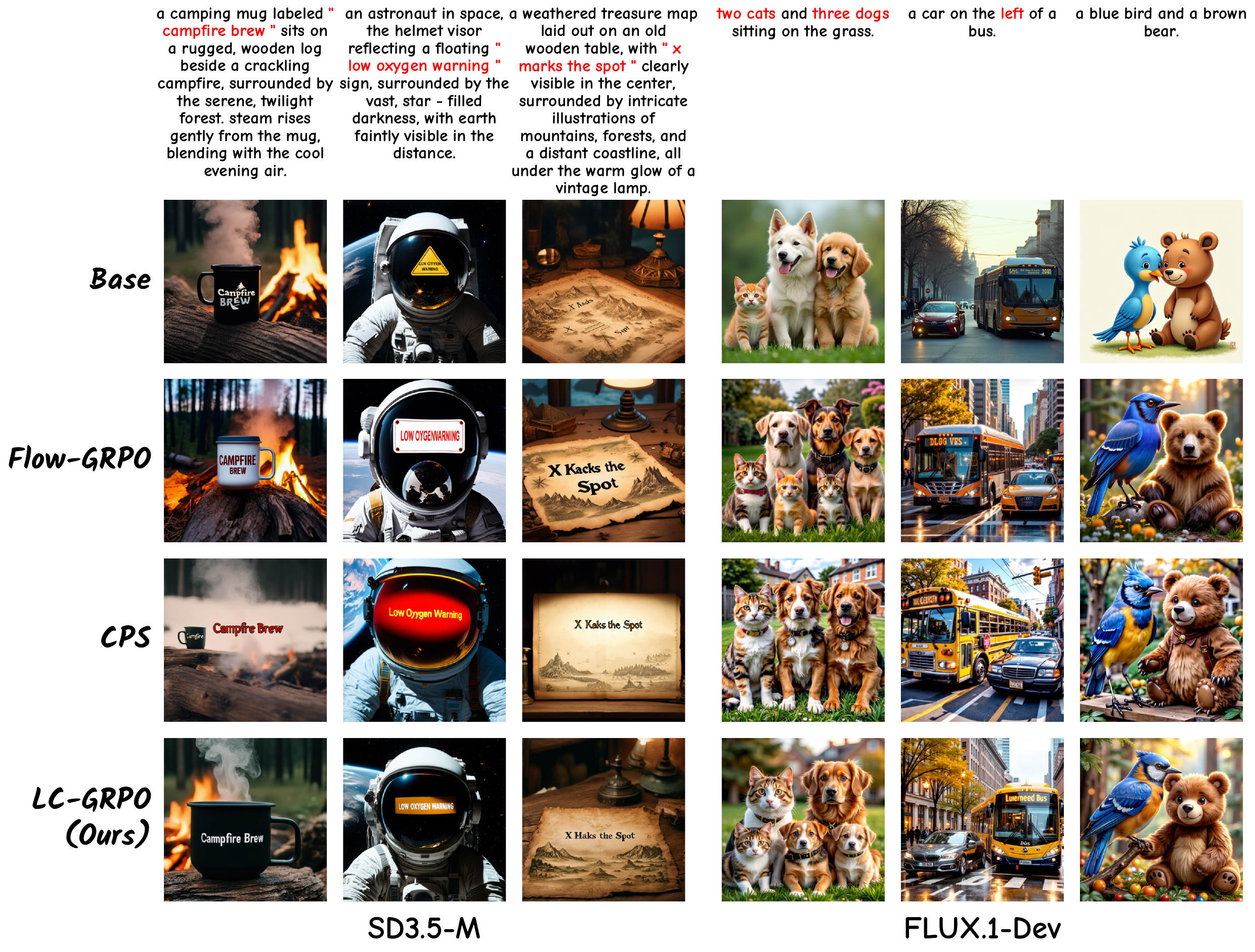}\\
	\caption{ \textbf{Qualitative Comparison}. Left: SD3.5-M-based models trained with OCR reward; Right: FLUX.1-Dev-based models trained with HPS-v2.1 and CLIP score. Prompts are taken from OCR and DrawBench, respectively.}\label{fig:qualitative imgs}
	\vspace{-.10in}
\end{figure}

\paragraph{Gap between Training and Evaluation.}
We examine the gap between rewards on RL training rollouts and those obtained during evaluation inference. Figure~\ref{fig:close_gap} compares training rewards, computed on rollouts, with evaluation rewards, computed with ODE Euler sampling on the evaluation dataset throughout training. Flow-GRPO tends to produce blurry, lower-quality rollouts, so its training rewards are often lower than those of cleaner evaluation samples. In contrast, CPS improves rollout quality but leads to a large train--evaluation gap, with high training rewards dropping during inference. \ours also improves rollout quality while maintaining a smaller train--test gap.

\begin{table*}[htbp]
    \centering
   \caption{\textbf{Evaluation Results on Video Generation.} HunyuanVideo is trained with the \textbf{Visual Quality} reward from VideoAlign. \textbf{Bold} indicates the best result. }
    \resizebox{0.9\linewidth}{!}{
        \begin{tabular}{lcccccccc}
            \toprule
            \multirow{3}{*}{\textbf{Model}}
            & \multicolumn{3}{c}{\textbf{VideoAlign}}
            & \multicolumn{5}{c}{\textbf{VBench}} \\
            \cmidrule(lr){2-4} \cmidrule(lr){5-9}
            & \makecell{Visual\\Quality}
            & \makecell{Motion\\Quality}
            & \makecell{Text\\Alignment}
            & \makecell{Total\\Score}
            & \makecell{Quality\\Score}
            & \makecell{Semantic\\Score}
            & \makecell{Dynamic\\Degree}
            & \makecell{Aesthetic\\Quality}\\ 
            \midrule 
            HunyuanVideo & -0.368 & - 0.351 & {1.247} & 78.74 & 80.94 & \textbf{69.96} & \textbf{52.8} & 61.28 \\ \midrule
            DanceGRPO & -0.205 & -0.279 & \textbf{1.312} &  78.85 & 81.18 & 69.51 & 12.5 & 64.03 \\
            CPS & -0.276 & \textbf{-0.178} & {1.048} & 76.86 & 79.70 & {65.49} & 19.4 & 59.02  \\
            \ours  & \textbf{0.063} & -0.223 & {1.151} & \textbf{79.10} & \textbf{81.92} & {67.83} & {45.8} & \textbf{64.70}\\ 
            \bottomrule
        \end{tabular}
    }
    \label{tab:video_res}
\end{table*}


\subsection{Ablation Studies}

In this section, we analyze the design choices of the Langevin correction step: 
\textbf{(a) Gradients through Langevin Correction.} Since the Langevin correction step is not used at test time, we study whether gradients should be propagated through it during training. As shown in Fig.~\ref{fig:ablation_grad_cfg}, backpropagating through the Langevin step leads to better performance and more stable training. We therefore keep this gradient path in all experiments. 
\textbf{(b) Classifier-Free Guidance.} Classifier-free guidance (CFG) is commonly used to improve sample quality, raising the question of whether it should also be applied in the Langevin correction step. However, the main role of Langevin correction is to introduce stochasticity and define a sampling distribution, rather than directly steering samples toward higher visual quality. We therefore ablate this design choice. As shown in Fig.~\ref{fig:ablation_grad_cfg}(b), applying Langevin correction without CFG achieves better performance on SD3.5-M while also reducing computation.

\begin{figure}[htbp]
  \centering
  \begin{minipage}[t]{0.45\textwidth}
    \centering
    \includegraphics[width=0.8\linewidth]{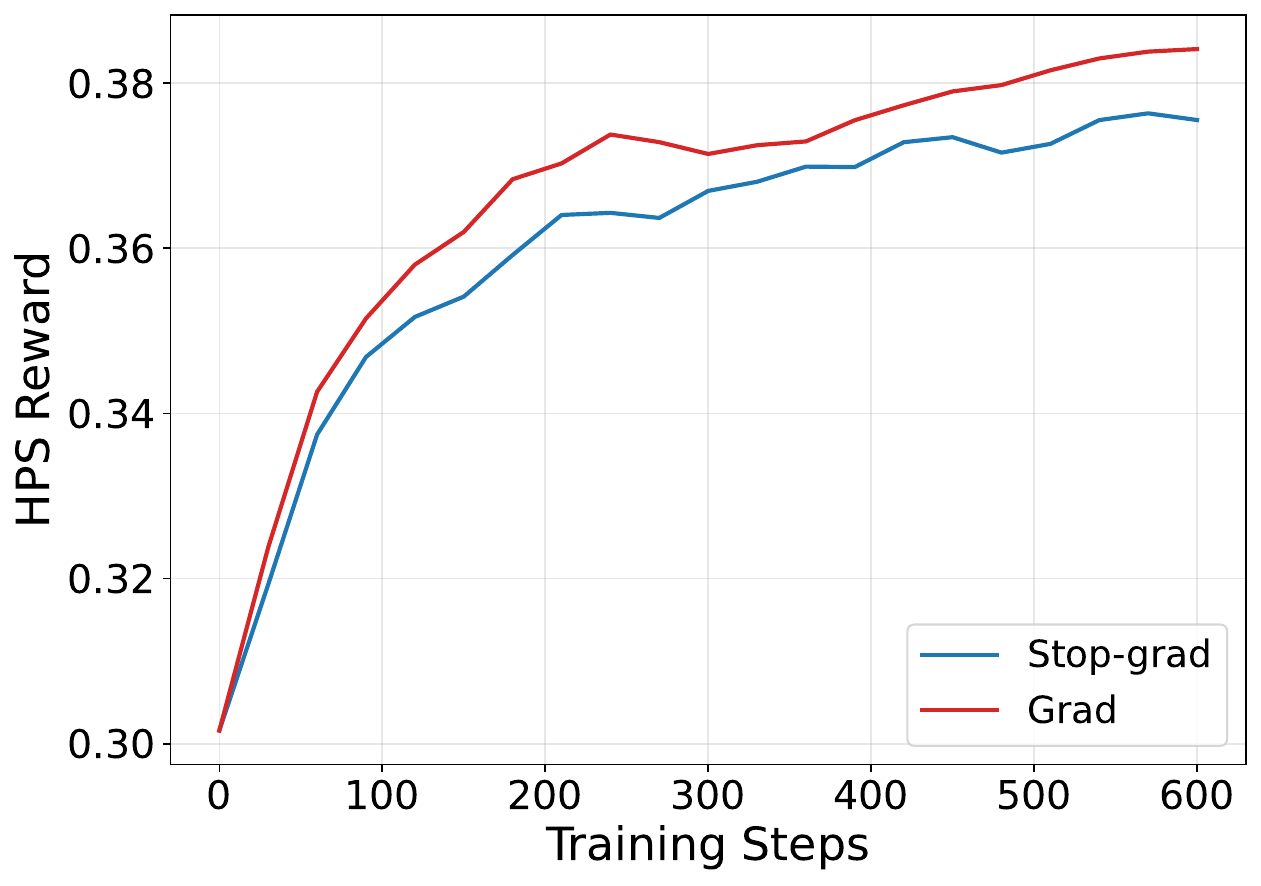}
    \caption*{\footnotesize (a) FLUX.1-Dev: gradients in Langevin step  }
  \end{minipage}
  \hfill
  \begin{minipage}[t]{0.45\textwidth}
    \centering
    \includegraphics[width=0.8\linewidth]{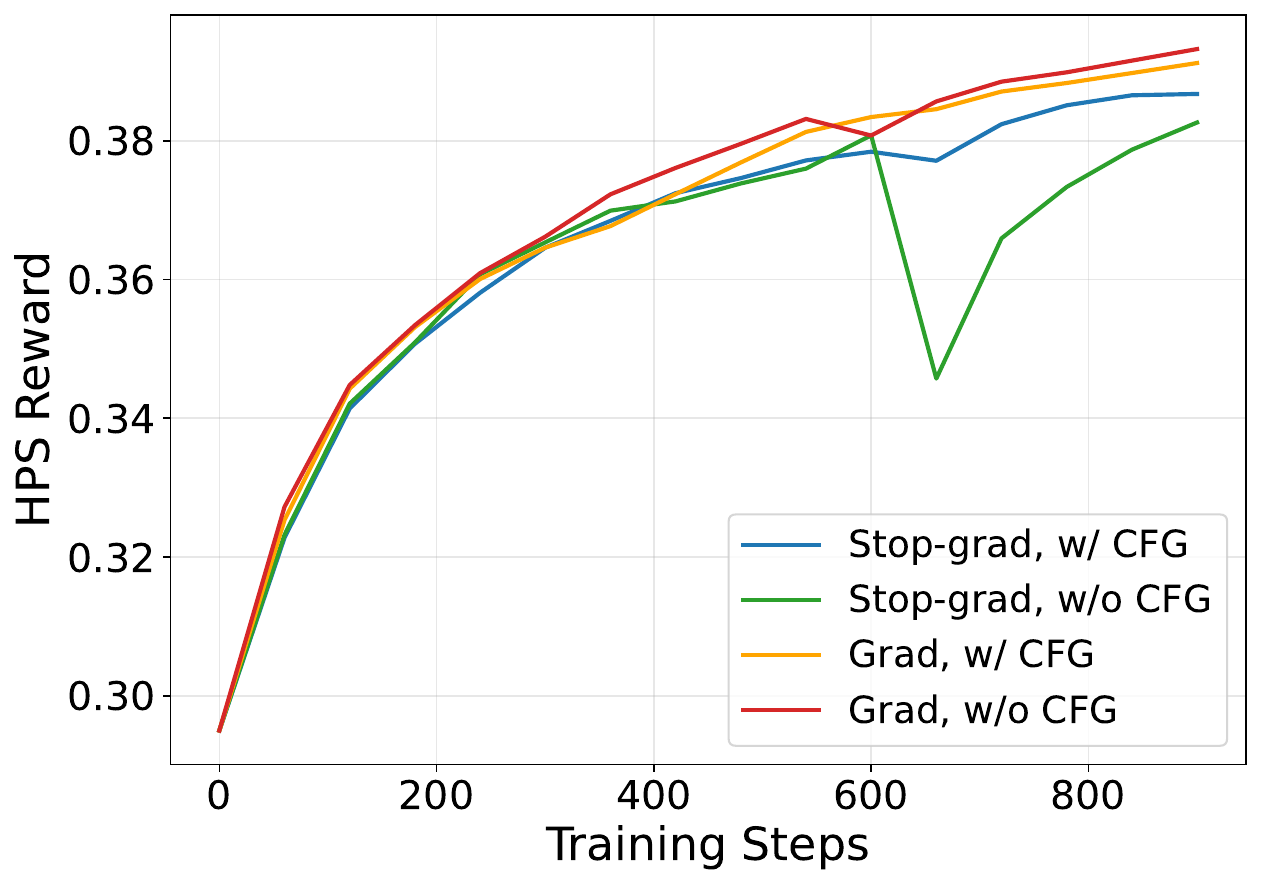}
    \caption*{\footnotesize (b) SD3.5-M: gradients and CFG in Langevin step  }
  \end{minipage}
  \caption{
  Effect of design choices in the Langevin correction step.
  }
  \vspace{-4mm}
  \label{fig:ablation_grad_cfg}
\end{figure}

\section{Conclusion}

We introduced \ours, a Langevin-corrected framework for online RL with flow-based generative models that reduces the mismatch between stochastic training rollouts and deterministic ODE inference. By combining inference-aligned ODE transport with a distribution-preserving correction, \ours enables exploration while maintaining accurate rollouts and tractable policy likelihoods. Our analysis shows that, at matched exploration strength, the proposed transition can achieve lower sampling error than standard SDE discretization. Experiments across text-to-image and text-to-video alignment tasks demonstrate consistent improvements in rollout quality and reward optimization, highlighting the benefit of aligning training-time sampling with inference-time generation.

\bibliography{references}

@article{liu2022flow,
  title={Flow straight and fast: Learning to generate and transfer data with rectified flow},
  author={Liu, Xingchao and Gong, Chengyue and Liu, Qiang},
  journal={arXiv preprint arXiv:2209.03003},
  year={2022}
}

@article{lipman2022flow,
  title={Flow matching for generative modeling},
  author={Lipman, Yaron and Chen, Ricky TQ and Ben-Hamu, Heli and Nickel, Maximilian and Le, Matt},
  journal={arXiv preprint arXiv:2210.02747},
  year={2022}
}

@inproceedings{black2024training,
  title={Training diffusion models with reinforcement learning},
  author={Black, Kevin and Janner, Michael and Du, Yilun and Kostrikov, Ilya and Levine, Sergey},
  booktitle={International Conference on Learning Representations},
  volume={2024},
  pages={4965--4987},
  year={2024}
}

@article{liu2026flow,
  title={Flow-grpo: Training flow matching models via online rl},
  author={Liu, Jie and Liu, Gongye and Liang, Jiajun and Li, Yangguang and Liu, Jiaheng and Wang, Xintao and Wan, Pengfei and Zhang, Di and Ouyang, Wanli},
  journal={Advances in neural information processing systems},
  volume={38},
  pages={40783--40818},
  year={2026}
}

@article{xue2025dancegrpo,
  title={Dancegrpo: Unleashing grpo on visual generation},
  author={Xue, Zeyue and Wu, Jie and Gao, Yu and Kong, Fangyuan and Zhu, Lingting and Chen, Mengzhao and Liu, Zhiheng and Liu, Wei and Guo, Qiushan and Huang, Weilin and others},
  journal={arXiv preprint arXiv:2505.07818},
  year={2025}
}

@article{otto2001geometry,
  title={The geometry of dissipative evolution equations: the porous medium equation},
  author={Otto, Felix},
  year={2001},
  publisher={Taylor \& Francis}
}

@article{wu2025qwen,
  title={Qwen-image technical report},
  author={Wu, Chenfei and Li, Jiahao and Zhou, Jingren and Lin, Junyang and Gao, Kaiyuan and Yan, Kun and Yin, Sheng-ming and Bai, Shuai and Xu, Xiao and Chen, Yilei and others},
  journal={arXiv preprint arXiv:2508.02324},
  year={2025}
}

@article{wan2025wan,
  title={Wan: Open and advanced large-scale video generative models},
  author={Wan, Team and Wang, Ang and Ai, Baole and Wen, Bin and Mao, Chaojie and Xie, Chen-Wei and Chen, Di and Yu, Feiwu and Zhao, Haiming and Yang, Jianxiao and others},
  journal={arXiv preprint arXiv:2503.20314},
  year={2025}
}

@inproceedings{esser2024scaling,
  title={Scaling rectified flow transformers for high-resolution image synthesis},
  author={Esser, Patrick and Kulal, Sumith and Blattmann, Andreas and Entezari, Rahim and M{\"u}ller, Jonas and Saini, Harry and Levi, Yam and Lorenz, Dominik and Sauer, Axel and Boesel, Frederic and others},
  booktitle={Forty-first international conference on machine learning},
  year={2024}
}

@misc{flux2024,
    author={Black Forest Labs},
    title={FLUX},
    year={2024},
    howpublished={\url{https://github.com/black-forest-labs/flux}},
}

@inproceedings{wu2023human,
  title={Human preference score: Better aligning text-to-image models with human preference},
  author={Wu, Xiaoshi and Sun, Keqiang and Zhu, Feng and Zhao, Rui and Li, Hongsheng},
  booktitle={Proceedings of the IEEE/CVF International Conference on Computer Vision},
  pages={2096--2105},
  year={2023}
}

@article{kirstain2023pick,
  title={Pick-a-pic: An open dataset of user preferences for text-to-image generation},
  author={Kirstain, Yuval and Polyak, Adam and Singer, Uriel and Matiana, Shahbuland and Penna, Joe and Levy, Omer},
  journal={Advances in neural information processing systems},
  volume={36},
  pages={36652--36663},
  year={2023}
}

@inproceedings{hessel2021clipscore,
  title={Clipscore: A reference-free evaluation metric for image captioning},
  author={Hessel, Jack and Holtzman, Ari and Forbes, Maxwell and Le Bras, Ronan and Choi, Yejin},
  booktitle={Proceedings of the 2021 conference on empirical methods in natural language processing},
  pages={7514--7528},
  year={2021}
}

@article{saharia2022photorealistic,
  title={Photorealistic text-to-image diffusion models with deep language understanding},
  author={Saharia, Chitwan and Chan, William and Saxena, Saurabh and Li, Lala and Whang, Jay and Denton, Emily L and Ghasemipour, Kamyar and Gontijo Lopes, Raphael and Karagol Ayan, Burcu and Salimans, Tim and others},
  journal={Advances in neural information processing systems},
  volume={35},
  pages={36479--36494},
  year={2022}
}

@article{xu2023imagereward,
  title={Imagereward: Learning and evaluating human preferences for text-to-image generation},
  author={Xu, Jiazheng and Liu, Xiao and Wu, Yuchen and Tong, Yuxuan and Li, Qinkai and Ding, Ming and Tang, Jie and Dong, Yuxiao},
  journal={Advances in Neural Information Processing Systems},
  volume={36},
  pages={15903--15935},
  year={2023}
}

@InProceedings{huang2023vbench,
     title={{VBench}: Comprehensive Benchmark Suite for Video Generative Models},
     author={Huang, Ziqi and He, Yinan and Yu, Jiashuo and Zhang, Fan and Si, Chenyang and Jiang, Yuming and Zhang, Yuanhan and Wu, Tianxing and Jin, Qingyang and Chanpaisit, Nattapol and Wang, Yaohui and Chen, Xinyuan and Wang, Limin and Lin, Dahua and Qiao, Yu and Liu, Ziwei},
     booktitle={Proceedings of the IEEE/CVF Conference on Computer Vision and Pattern Recognition},
     year={2024}
 }

@article{liu2026improving,
  title={Improving video generation with human feedback},
  author={Liu, Jie and Liu, Gongye and Liang, Jiajun and Yuan, Ziyang and Liu, Xiaokun and Zheng, Mingwu and Wu, Xiele and Wang, Qiulin and Xia, Menghan and Wang, Xintao and others},
  journal={Advances in Neural Information Processing Systems},
  volume={38},
  pages={82155--82192},
  year={2026}
}

@article{kong2024hunyuanvideo,
  title={Hunyuanvideo: A systematic framework for large video generative models},
  author={Kong, Weijie and Tian, Qi and Zhang, Zijian and Min, Rox and Dai, Zuozhuo and Zhou, Jin and Xiong, Jiangfeng and Li, Xin and Wu, Bo and Zhang, Jianwei and others},
  journal={arXiv preprint arXiv:2412.03603},
  year={2024}
}

@article{wang2025coefficients,
  title={Coefficients-Preserving Sampling for Reinforcement Learning with Flow Matching},
  author={Wang, Feng and Yu, Zihao},
  journal={arXiv preprint arXiv:2509.05952},
  year={2025}
}

@article{song2020denoising,
  title={Denoising diffusion implicit models},
  author={Song, Jiaming and Meng, Chenlin and Ermon, Stefano},
  journal={arXiv preprint arXiv:2010.02502},
  year={2020}
}

@article{wang2024vidprom,
  title={Vidprom: A million-scale real prompt-gallery dataset for text-to-video diffusion models},
  author={Wang, Wenhao and Yang, Yi},
  journal={Advances in Neural Information Processing Systems},
  volume={37},
  pages={65618--65642},
  year={2024}
}

@article{song2020score,
  title={Score-based generative modeling through stochastic differential equations},
  author={Song, Yang and Sohl-Dickstein, Jascha and Kingma, Diederik P and Kumar, Abhishek and Ermon, Stefano and Poole, Ben},
  journal={arXiv preprint arXiv:2011.13456},
  year={2020}
}

@article{efron2011tweedie,
  title={Tweedie’s formula and selection bias},
  author={Efron, Bradley},
  journal={Journal of the American Statistical Association},
  volume={106},
  number={496},
  pages={1602--1614},
  year={2011},
  publisher={Taylor \& Francis}
}

@book{kloeden2012numerical,
  title={Numerical solution of SDE through computer experiments},
  author={Kloeden, Peter Eris and Platen, Eckhard and Schurz, Henri},
  year={2012},
  publisher={Springer Science \& Business Media}
}

@article{zheng2026manifold,
  title={Manifold-Aware Exploration for Reinforcement Learning in Video Generation},
  author={Zheng, Mingzhe and Kong, Weijie and Wu, Yue and Jiang, Dengyang and Ma, Yue and He, Xuanhua and Lin, Bin and Gong, Kaixiong and Zhong, Zhao and Bo, Liefeng and others},
  journal={arXiv preprint arXiv:2603.21872},
  year={2026}
}

@inproceedings{clark2024directly,
  title={Directly fine-tuning diffusion models on differentiable rewards},
  author={Clark, Kevin and Vicol, Paul and Swersky, Kevin and Fleet, David},
  booktitle={International Conference on Learning Representations},
  volume={2024},
  pages={4793--4822},
  year={2024}
}

@article{prabhudesai2023aligning,
  title={Aligning text-to-image diffusion models with reward backpropagation},
  author={Prabhudesai, Mihir and Goyal, Anirudh and Pathak, Deepak and Fragkiadaki, Katerina},
  year={2023}
}

@article{lee2023aligning,
  title={Aligning text-to-image models using human feedback},
  author={Lee, Kimin and Liu, Hao and Ryu, Moonkyung and Watkins, Olivia and Du, Yuqing and Boutilier, Craig and Abbeel, Pieter and Ghavamzadeh, Mohammad and Gu, Shixiang Shane},
  journal={arXiv preprint arXiv:2302.12192},
  year={2023}
}

@article{dong2023raft,
  title={Raft: Reward ranked finetuning for generative foundation model alignment},
  author={Dong, Hanze and Xiong, Wei and Goyal, Deepanshu and Zhang, Yihan and Chow, Winnie and Pan, Rui and Diao, Shizhe and Zhang, Jipeng and Shum, Kashun and Zhang, Tong},
  journal={arXiv preprint arXiv:2304.06767},
  year={2023}
}

@article{rafailov2023direct,
  title={Direct preference optimization: Your language model is secretly a reward model},
  author={Rafailov, Rafael and Sharma, Archit and Mitchell, Eric and Manning, Christopher D and Ermon, Stefano and Finn, Chelsea},
  journal={Advances in neural information processing systems},
  volume={36},
  pages={53728--53741},
  year={2023}
}

@inproceedings{wallace2024diffusion,
  title={Diffusion model alignment using direct preference optimization},
  author={Wallace, Bram and Dang, Meihua and Rafailov, Rafael and Zhou, Linqi and Lou, Aaron and Purushwalkam, Senthil and Ermon, Stefano and Xiong, Caiming and Joty, Shafiq and Naik, Nikhil},
  booktitle={Proceedings of the IEEE/CVF Conference on Computer Vision and Pattern Recognition},
  pages={8228--8238},
  year={2024}
}

@inproceedings{yang2024using,
  title={Using human feedback to fine-tune diffusion models without any reward model},
  author={Yang, Kai and Tao, Jian and Lyu, Jiafei and Ge, Chunjiang and Chen, Jiaxin and Shen, Weihan and Zhu, Xiaolong and Li, Xiu},
  booktitle={Proceedings of the IEEE/CVF Conference on Computer Vision and Pattern Recognition},
  pages={8941--8951},
  year={2024}
}

@article{zheng2025diffusionnft,
  title={Diffusionnft: Online diffusion reinforcement with forward process},
  author={Zheng, Kaiwen and Chen, Huayu and Ye, Haotian and Wang, Haoxiang and Zhang, Qinsheng and Jiang, Kai and Su, Hang and Ermon, Stefano and Zhu, Jun and Liu, Ming-Yu},
  journal={arXiv preprint arXiv:2509.16117},
  year={2025}
}

@article{fan2023dpok,
  title={Dpok: Reinforcement learning for fine-tuning text-to-image diffusion models},
  author={Fan, Ying and Watkins, Olivia and Du, Yuqing and Liu, Hao and Ryu, Moonkyung and Boutilier, Craig and Abbeel, Pieter and Ghavamzadeh, Mohammad and Lee, Kangwook and Lee, Kimin},
  journal={Advances in Neural Information Processing Systems},
  volume={36},
  pages={79858--79885},
  year={2023}
}

@article{li2025mixgrpo,
  title={Mixgrpo: Unlocking flow-based grpo efficiency with mixed ode-sde},
  author={Li, Junzhe and Cui, Yutao and Huang, Tao and Ma, Yinping and Fan, Chun and Cheng, Yiming and Yang, Miles and Zhong, Zhao and Bo, Liefeng},
  journal={arXiv preprint arXiv:2507.21802},
  year={2025}
}

@book{ambrosio2005gradient,
  title={Gradient flows: in metric spaces and in the space of probability measures},
  author={Ambrosio, Luigi and Gigli, Nicola and Savar{\'e}, Giuseppe},
  year={2005},
  publisher={Springer}
}

@article{qi2025defeating,
  title={Defeating the training-inference mismatch via fp16},
  author={Qi, Penghui and Liu, Zichen and Zhou, Xiangxin and Pang, Tianyu and Du, Chao and Lee, Wee Sun and Lin, Min},
  journal={arXiv preprint arXiv:2510.26788},
  year={2025}
}

@article{qiu2026fp8,
  title={FP8-RL: A Practical and Stable Low-Precision Stack for LLM Reinforcement Learning},
  author={Qiu, Zhaopeng and Yu, Shuang and Zhang, Jingqi and Zhang, Shuai and Huang, Xue and Yang, Jingyi and Lai, Junjie},
  journal={arXiv preprint arXiv:2601.18150},
  year={2026}
}

@article{wasti2025no,
  title={No More Train-Inference Mismatch: Bitwise Consistent On-Policy Reinforcement Learning with vLLM and TorchTitan. 2025},
  author={Wasti, Bram and Ye, Wentao and Rao, Teja and Goin, Michael and others},
  journal={URl: https://blog. vllm. ai/2025/11/10/bitwise-consistent-train-inference. html}
}

@article{yao2025your,
  title={Your efficient rl framework secretly brings you off-policy rl training, August 2025},
  author={Yao, Feng and Liu, Liyuan and Zhang, Dinghuai and Dong, Chengyu and Shang, Jingbo and Gao, Jianfeng},
  journal={URL https://fengyao. notion. site/off-policy-rl},
  year={2025}
}

@article{liu2025speed,
  title={When speed kills stability: Demystifying rl collapse from the training-inference mismatch},
  author={Liu, Jiacai and Li, Yingru and Fu, Yuqian and Wang, Jiawei and Liu, Qian and Shen, Yu},
  journal={Notion Blog},
  year={2025}
}

@article{liang2026mirage,
  title={The Mirage of Optimizing Training Policies: Monotonic Inference Policies as the Real Objective for LLM Reinforcement Learning},
  author={Liang, Jing and Tang, Hongyao and Ma, Yi and He, Yancheng and Wang, Weixun and Li, Xiaoyang and Huang, Ju and Su, Wenbo and Liu, Jinyi and Zheng, Yan and others},
  journal={arXiv preprint arXiv:2606.29526},
  year={2026}
}

@article{dhariwal2021diffusion,
  title={Diffusion models beat gans on image synthesis},
  author={Dhariwal, Prafulla and Nichol, Alexander},
  journal={Advances in neural information processing systems},
  volume={34},
  pages={8780--8794},
  year={2021}
}

@article{chung2022diffusion,
  title={Diffusion posterior sampling for general noisy inverse problems},
  author={Chung, Hyungjin and Kim, Jeongsol and Mccann, Michael T and Klasky, Marc L and Ye, Jong Chul},
  journal={arXiv preprint arXiv:2209.14687},
  year={2022}
}

@article{guo2024gradient,
  title={Gradient Guidance for Diffusion Models: An Optimization Perspective},
  author={Guo, Yingqing and Yuan, Hui and Yang, Yukang and Chen, Minshuo and Wang, Mengdi},
  journal={arXiv preprint arXiv:2404.14743},
  year={2024}
}

@article{guo2026training,
  title={Training-free guidance beyond differentiability: Scalable path steering with tree search in diffusion and flow models},
  author={Guo, Yingqing and Yang, Yukang and Yuan, Hui and Wang, Mengdi},
  journal={Advances in Neural Information Processing Systems},
  volume={38},
  pages={73343--73384},
  year={2026}
}

@article{deng2026densegrpo,
  title={Densegrpo: From sparse to dense reward for flow matching model alignment},
  author={Deng, Haoyou and Yan, Keyu and Mao, Chaojie and Wang, Xiang and Liu, Yu and Gao, Changxin and Sang, Nong},
  journal={arXiv preprint arXiv:2601.20218},
  year={2026}
}

@inproceedings{savani2026stepwise,
  title={Stepwise Credit Assignment for GRPO on Flow-Matching Models},
  author={Savani, Yash and Kveton, Branislav and Liu, Yuchen and Wang, Yilin and Shi, Jing and Mukherjee, Subhojyoti and Vlassis, Nikos and Singh, Krishna Kumar},
  booktitle={Proceedings of the IEEE/CVF Conference on Computer Vision and Pattern Recognition},
  pages={42007--42017},
  year={2026}
}

@article{ding2025treegrpo,
  title={TreeGRPO: Tree-Advantage GRPO for Online RL Post-Training of Diffusion Models},
  author={Ding, Zheng and Ye, Weirui},
  journal={arXiv preprint arXiv:2512.08153},
  year={2025}
}

@article{yuan2023reward,
  title={Reward-directed conditional diffusion: Provable distribution estimation and reward improvement},
  author={Yuan, Hui and Huang, Kaixuan and Ni, Chengzhuo and Chen, Minshuo and Wang, Mengdi},
  journal={Advances in Neural Information Processing Systems},
  volume={36},
  pages={60599--60635},
  year={2023}
}

@article{xu2026qwen,
  title={Qwen-Image-2.0-RL Technical Report},
  author={Xu, Yixian and Gao, Kaiyuan and Chen, Yuxiang and Chen, Yilei and Tang, Zecheng and Liu, Zihao and Zhou, Zikai and Li, Deqing and Meng, Hao and Cao, Kuan and others},
  journal={arXiv preprint arXiv:2606.27608},
  year={2026}
}

@article{gao2025seedream,
  title={Seedream 3.0 technical report},
  author={Gao, Yu and Gong, Lixue and Guo, Qiushan and Hou, Xiaoxia and Lai, Zhichao and Li, Fanshi and Li, Liang and Lian, Xiaochen and Liao, Chao and Liu, Liyang and others},
  journal={arXiv preprint arXiv:2504.11346},
  year={2025}
}

@article{gao2025seedance,
  title={Seedance 1.0: Exploring the boundaries of video generation models},
  author={Gao, Yu and Guo, Haoyuan and Hoang, Tuyen and Huang, Weilin and Jiang, Lu and Kong, Fangyuan and Li, Huixia and Li, Jiashi and Li, Liang and Li, Xiaojie and others},
  journal={arXiv preprint arXiv:2506.09113},
  year={2025}
}

@article{wu2025hunyuanvideo,
  title={Hunyuanvideo 1.5 technical report},
  author={Wu, Bing and Zou, Chang and Li, Changlin and Huang, Duojun and Yang, Fang and Tan, Hao and Peng, Jack and Wu, Jianbing and Xiong, Jiangfeng and Jiang, Jie and others},
  journal={arXiv preprint arXiv:2511.18870},
  year={2025}
}
\bibliographystyle{iclr_2027_submission/iclr2027_conference}

\appendix

\section{Theoretical Analysis}\label{appx:theory}

We focus on a single denoising step of the flow model, running from time $t$ to $r := t - h$ with step size $h > 0$. The probability flow ODE is
\begin{equation}\label{eq:pf-ode}
    \mathrm{d}\boldsymbol{x}_u = \boldsymbol{v}_u\rbr{\boldsymbol{x}_u}\,\mathrm{d}u,
\end{equation}
with marginals $p_u$ for $u \in [r, t]$.

Under this assumption, we show that a single Langevin step provably improves upon the Euler step. This result is stated formally as Theorem~\ref{thm:corrector-improves}, with the proof given in Sec.~\ref{proof:corrector-improves}.

\begin{assumption}\label{ass:slc}
Let $f := \log p_r$. Then $-L I_d \preceq \nabla^2 f \preceq -\alpha I_d$ for some $0 < \alpha \le L$; that is, $f$ is $\alpha$-strongly concave and $\nabla f$ is $L$-Lipschitz.
\end{assumption}
\begin{theorem}[Langevin Step Improves the Euler Step]\label{thm:corrector-improves}
Let Assumption~\ref{ass:slc} hold. Starting from $\boldsymbol{x} \sim q_t$ at time $t$, take one backward Euler step of the probability flow ODE Eq.~\ref{eq:pf-ode} with step size $h > 0$,
\begin{equation*}
    \boldsymbol{x}_{\mathrm{ode}} = \boldsymbol{x} - h\,\boldsymbol{v}_t(\boldsymbol{x}),
\end{equation*}
and denote the resulting error by $\varepsilon := \W\big(\law(\boldsymbol{x}_{\mathrm{ode}}),\, p_r\big)$, assumed positive. Then follow with one Langevin correction (LC) step targeting $p_r$,
\begin{equation*}
    \boldsymbol{x}_{\mathrm{lc}} = \boldsymbol{x}_{\mathrm{ode}} + \epsilon\, \nabla \log p_r(\boldsymbol{x}_{\mathrm{ode}}) + \sqrt{2\epsilon}\,\xi,
    \qquad \xi \sim \mathcal{N}(0, I_d), \quad \xi \perp \boldsymbol{x}_{\mathrm{ode}} .
\end{equation*}
If the step size satisfies
\begin{equation}\label{eq:eta-cond}
    \epsilon \;\le\; \frac{\alpha}{L^2}
    \qquad \text{and} \qquad
    \epsilon \;<\; \frac{9}{100}\cdot\frac{\alpha^2 \varepsilon^2}{L^2 d},
\end{equation}
then the Langevin step strictly improves upon the Euler step:
\begin{equation*}\label{eq:strict-improve}
\W\big(\law(\boldsymbol{x}_{\mathrm{lc}}),\, p_r\big)
    \;<\; \W\big(\law(\boldsymbol{x}_{\mathrm{ode}}),\, p_r\big).
\end{equation*}
\end{theorem}

Theorem~\ref{thm:corrector-improves} shows that the corrector improves on the deterministic Euler step. We now compare it with the standard stochastic alternative: one Euler--Maruyama step of the reverse SDE. Denote $\boldsymbol{s}_u := \nabla \log p_u$ for the score. Given $\boldsymbol{x} \sim p_t$ and $\xi \sim \mathcal{N}(0, I_d)$, the SDE step at noise level $\sigma > 0$ is
\begin{equation}\label{eq:sde-em}
    \boldsymbol{x}_{\mathrm{sde}} = \boldsymbol{x} - h\,\boldsymbol{v}_t(\boldsymbol{x}) + \frac{\sigma^2 h}{2}\,\boldsymbol{s}_t(\boldsymbol{x}) + \sigma\sqrt{h}\,\xi,
\end{equation}
while the Langevin correction step is exactly that of Theorem~\ref{thm:corrector-improves}. We further define vector fields at time $t$:
\begin{equation}\label{eq:def vectors}
    \dot{\boldsymbol{v}} := \partial_t \boldsymbol{v}_t + \nabla \boldsymbol{v}_t\, \boldsymbol{v}_t,
    \qquad
    \dot{\boldsymbol{s}} := \partial_t \boldsymbol{s}_t + \nabla \boldsymbol{s}_t\, \boldsymbol{v}_t,
    \qquad
    \boldsymbol{G} := \nabla\Big(\tfrac{1}{2}\nbr{\boldsymbol{s}_t}^2 + \Delta \log p_t\Big),
\end{equation}
Here, $\dot{\boldsymbol{v}}$ and $\dot{\boldsymbol{s}}$ are the rates of change of the velocity and score along the flow, while $\boldsymbol{G}$ is the direction of the bias introduced by a single Langevin step.

Under the following regularity assumption, we can compute the exact leading-order $\W$ error of both the Langevin correction and the SDE step, and show that the Langevin correction step is more accurate under an explicit condition (Theorem~\ref{thm:main}); see Sec.~\ref{proof:main} for the proof.

\begin{assumption}\label{ass:reg}
Let $\boldsymbol{v}, \log p \in C^{2,4}\big([r,t] \times \mathbb{R}^d\big)$, with all derivatives of polynomial growth, and $p_t$ has finite moments of all orders.
\end{assumption}

\begin{theorem}[Comparison with the SDE Step]\label{thm:main}
Let Assumption~\ref{ass:reg} hold and $\epsilon \leq \Lambda h$ for a fixed constant $\Lambda > 0$.  With $\dot{\boldsymbol{v}}, \dot{\boldsymbol{s}}, \boldsymbol{G}$ as defined in Eq.~\ref{eq:def vectors}, we have
\begin{equation}\label{eq:errors}
\begin{aligned}
    \W\big(\law(\boldsymbol{x}_{\mathrm{lc}}),\, p_{r}\big) &= \frac{h^2}{2}\,\nbr{\dot{\boldsymbol{v}} - \rbr{\tfrac{\epsilon}{h}}^{2} \boldsymbol{G}}_{p_t} + o(h^2), \\
    \W\big(\law(\boldsymbol{x}_{\mathrm{sde}}),\, p_{r}\big) &= \frac{h^2}{2}\,\nbr{\dot{\boldsymbol{v}} - \sigma^2 \dot{\boldsymbol{s}} - \tfrac{\sigma^4}{4} \boldsymbol{G}}_{p_t} + o(h^2).
\end{aligned}
\end{equation}
In particular, at $\epsilon = \sigma^2 h / 2$ matching the noise level, let
$\boldsymbol{E} := \dot{\boldsymbol{v}} - \tfrac{\sigma^4}{4}\boldsymbol{G}$, if the following holds:
\begin{equation}\label{eq:lc_sde_cond_matched_noise}
\big\langle \boldsymbol{E}, \dot{\boldsymbol{s}} \big\rangle_{p_t} < \frac{\sigma^2}{2} \nbr{\dot{\boldsymbol{s}}}_{p_t}^{2},
\end{equation} 
 Langevin correction step is strictly more accurate:
\begin{equation*}
    \W\big(\law(\boldsymbol{x}_{\mathrm{lc}}),\, p_{r}\big) \;\ < \; \W\big(\law(\boldsymbol{x}_{\mathrm{sde}}),\, p_{r}\big),
\end{equation*}
for all sufficiently small $h$.
\end{theorem}

\subsection{Proof of Theorem~\ref{thm:corrector-improves}} \label{proof:corrector-improves}

\begin{proof}
Applying Lemma~\ref{lem:langevin} at time $r$ with initial law $q = q_r$ gives the upper bound:
\begin{equation*}
    \W\big(\law(\boldsymbol{x}_{\mathrm{lc}}),\, p_r\big)
    \;\le\; \Big(1 - \frac{\alpha\epsilon}{2}\Big)\,\varepsilon + \frac{5}{3}\,L\sqrt{d}\,\epsilon^{3/2}.
\end{equation*}
The second condition in Eq.~\ref{eq:eta-cond} is equivalent to $\sqrt{\epsilon} < \frac{3\alpha\varepsilon}{10\,L\sqrt{d}}$, hence
\begin{equation*}
    \frac{5}{3}\,L\sqrt{d}\,\epsilon^{3/2}
    \;<\; \frac{5}{3}\,L\sqrt{d}\cdot\frac{3\alpha\varepsilon}{10\,L\sqrt{d}}\cdot\epsilon
    \;=\; \frac{\alpha\epsilon}{2}\,\varepsilon ,
\end{equation*}
so the discretization error is strictly dominated by the contraction gain, and
\begin{equation*}
    \W\big(\law(\boldsymbol{x}_{\mathrm{lc}}),\, p_r\big)
    \;<\; \Big(1 - \frac{\alpha\epsilon}{2}\Big)\,\varepsilon + \frac{\alpha\epsilon}{2}\,\varepsilon
    \;=\; \varepsilon .
\end{equation*}
Therefore, we complete the proof.
\end{proof}

\begin{lemma}[Langevin Corrector]\label{lem:langevin}
Let Assumption~\ref{ass:slc} hold and given $\boldsymbol{z} \sim q$, langevin correction step:
\begin{equation}
    \boldsymbol{z}^{+} = \boldsymbol{z} + \epsilon \nabla \log p_r (\boldsymbol{z})  + \sqrt{2\epsilon}\,{\xi}, \quad {\xi} \sim \mathcal{N}(0, I_d)
\end{equation}
with $\epsilon \;\le\; {\alpha}/{L^2}$.
Then,
\begin{equation}\label{eq:corrector-bound}
  W_2\big(\mathrm{law}(\boldsymbol{z}^{+}),\, p_r\big)
  \;\le\; \Big(1 - \frac{\alpha\epsilon}{2}\Big)\, W_2(q,\, p_r)
  \;+\; \frac{5}{3}\, L \sqrt{d}\;\epsilon^{3/2}.
\end{equation}
\end{lemma}

\begin{proof}
Denote $f := \log p_r$, so that the corrector step reads $\boldsymbol{z}^{+} = \boldsymbol{z} + \epsilon \nabla f(\boldsymbol{z}) + \sqrt{2\epsilon}\,\xi$. Let $(\boldsymbol{z}, \boldsymbol{w}_0)$ be an optimal $W_2$ coupling of $(q, p_r)$, and let $(B_\tau)_{\tau \geq 0}$ be a standard Brownian motion independent of $(\boldsymbol{z}, \boldsymbol{w}_0)$. Since $\nabla f$ is $L$-Lipschitz by Assumption~\ref{ass:slc}, the Langevin dynamics
\begin{equation*}
    \mathrm{d}\boldsymbol{w}_\tau = \nabla f(\boldsymbol{w}_\tau)\,\mathrm{d}\tau + \sqrt{2}\,\mathrm{d}B_\tau, \qquad \tau \in [0, \epsilon], \quad \boldsymbol{w}_0 \sim p_r,
\end{equation*}
admits a unique strong solution with $p_r \propto e^{f}$ as its stationary distribution, so $\boldsymbol{w}_\epsilon \sim p_r$. Let $\xi := B_\epsilon / \sqrt{\epsilon} \sim \mathcal{N}(0, I_d)$, which is independent of $\boldsymbol{z}$. Therefore, we have $(\boldsymbol{z}^{+}, \boldsymbol{w}_\epsilon)$ a coupling of $\mathrm{law}(\boldsymbol{z}^{+})$ and $p_r$, and thus
\begin{equation*}
    W_2\big(\mathrm{law}(\boldsymbol{z}^{+}),\, p_r\big) \leq \nbr{\boldsymbol{z}^{+} - \boldsymbol{w}_\epsilon}_{L^2}.
\end{equation*}

We first decompose the difference $\boldsymbol{z}^{+} - \boldsymbol{w}_\epsilon$. Writing $\boldsymbol{w}_\epsilon = \boldsymbol{w}_0 + \int_0^\epsilon \nabla f(\boldsymbol{w}_\tau)\,\mathrm{d}\tau + \sqrt{2}\,B_\epsilon$ and subtracting it from the corrector step, the Brownian terms cancel and
\begin{equation*}
    \boldsymbol{z}^{+} - \boldsymbol{w}_\epsilon
    = \underbrace{\rbr{\boldsymbol{z} - \boldsymbol{w}_0} + \epsilon \big(\nabla f(\boldsymbol{z}) - \nabla f(\boldsymbol{w}_0)\big)}_{=:\,A}
    + \underbrace{\int_0^\epsilon \big(\nabla f(\boldsymbol{w}_0) - \nabla f(\boldsymbol{w}_\tau)\big)\,\mathrm{d}\tau}_{=:\,B},
\end{equation*}
where $A$ compares the two drifts at the initial points, and $B$ is the local discretization error.

Next, we show that $A$ contracts. Expanding the square and using $\langle \nabla f(\boldsymbol{z}) - \nabla f(\boldsymbol{w}_0),\, \boldsymbol{z} - \boldsymbol{w}_0 \rangle \leq -\alpha \nbr{\boldsymbol{z} - \boldsymbol{w}_0}^2$ together with $\nbr{\nabla f(\boldsymbol{z}) - \nabla f(\boldsymbol{w}_0)} \leq L \nbr{\boldsymbol{z} - \boldsymbol{w}_0}$, where the strong log-concavity and the Lipschitz property of $\nabla f$ are from Assumption~\ref{ass:slc}, we have
\begin{equation*}
    \nbr{A}^2 \leq \big(1 - 2\alpha\epsilon + L^2\epsilon^2\big)\nbr{\boldsymbol{z} - \boldsymbol{w}_0}^2 \leq (1 - \alpha\epsilon)\nbr{\boldsymbol{z} - \boldsymbol{w}_0}^2,
\end{equation*}
where the last inequality uses $L^2\epsilon \leq \alpha$; note also $\alpha\epsilon \leq \alpha^2/L^2 \leq 1$, so the right-hand side is nonnegative. Since $\sqrt{1-x} \leq 1 - x/2$ for $x \in [0,1]$ and the coupling $(\boldsymbol{z}, \boldsymbol{w}_0)$ is optimal, it holds
\begin{equation*}
    \nbr{A}_{L^2} \leq \Big(1 - \frac{\alpha\epsilon}{2}\Big)\, W_2(q,\, p_r).
\end{equation*}

Then, we bound the local error $B$. Since Langevin dynamics is stationary, $\boldsymbol{w}_r \sim p_r$ for every $r$. Using $(a+b)^2 \leq 2a^2 + 2b^2$, the Cauchy--Schwarz inequality on the drift integral, Lemma~\ref{sub:moment}, and $\mathbb{E}\,\nbr{\sqrt{2}\,B_\tau}^2 = 2\tau d$, we obtain
\begin{equation*}
    \mathbb{E}\,\nbr{\boldsymbol{w}_\tau - \boldsymbol{w}_0}^2
    \leq 2\tau \int_0^\tau \mathbb{E}\,\nbr{\nabla f(\boldsymbol{w}_r)}^2\,\mathrm{d}r + 4\tau d
    \leq 2Ld\,\tau^2 + 4d \, \tau.
\end{equation*}
By the Cauchy--Schwarz inequality on the time integral and the Lipschitz property of $\nabla f$,
\begin{equation*}
    \mathbb{E}\,\nbr{B}^2
    \leq \epsilon \int_0^\epsilon \mathbb{E}\,\nbr{\nabla f(\boldsymbol{w}_0) - \nabla f(\boldsymbol{w}_\tau)}^2\,\mathrm{d}\tau
    \leq \epsilon L^2 \int_0^\epsilon \mathbb{E}\,\nbr{\boldsymbol{w}_\tau - \boldsymbol{w}_0}^2\,\mathrm{d}\tau
    \leq \epsilon L^2 \Big(\frac{2Ld\,\epsilon^3}{3} + 2d\,\epsilon^2\Big).
\end{equation*}
Since $\epsilon \leq \alpha/L^2 \leq 1/L$, we have $L\epsilon \leq 1$ and hence
\begin{equation*}
    \mathbb{E}\,\nbr{B}^2 \leq \frac{8}{3}\,L^2 d\,\epsilon^3,
    \qquad
    \nbr{B}_{L^2} \leq \sqrt{\tfrac{8}{3}}\,L\sqrt{d}\,\epsilon^{3/2} \leq \frac{5}{3}\,L\sqrt{d}\,\epsilon^{3/2}.
\end{equation*}

Finally, combining the above bounds yields
\begin{equation*}
    W_2\big(\mathrm{law}(\boldsymbol{z}^{+}),\, p_r\big)
    \leq \nbr{A}_{L^2} + \nbr{B}_{L^2}
    \leq \Big(1 - \frac{\alpha\epsilon}{2}\Big)\, W_2(q,\, p_r) + \frac{5}{3}\,L\sqrt{d}\;\epsilon^{3/2}.
    \qedhere
\end{equation*}
\end{proof}

\subsection{Proof of Theorem~\ref{thm:main} }\label{proof:main}

\begin{proof}
 For a law $q$ and a smooth test function $f$, define $\delta(f;q) := \int f\,\mathrm{d}q - \int f\,\mathrm{d}p_{r}$. For brevity, we write $\delta_{\mathrm{lc}}(f) := \delta(f; \law(\boldsymbol{x}_{\mathrm{lc}}))$ and $\delta_{\mathrm{sde}}(f) := \delta(f; \law(\boldsymbol{x}_{\mathrm{sde}}))$.

We first expand the Langevin correction step weakly. Let $f_\epsilon(\boldsymbol{y}) := \mathbb{E}_\xi\, f\big(\boldsymbol{y} + \epsilon\, \boldsymbol{s}_{r}(\boldsymbol{y}) + \sqrt{2\epsilon}\,\xi\big)$ denote the expected value of $f$ after one corrector step from $\boldsymbol{y}$, so that $\mathbb{E}\, f(\boldsymbol{x}_{\mathrm{lc}}) = \mathbb{E}_{p_t} f_\epsilon(\boldsymbol{x}_{\mathrm{ode}})$ and
\begin{equation}\label{eq:decomp}
    \delta_{\mathrm{lc}}(f)
    = \underbrace{\mathbb{E}_{p_t} f_\epsilon(\boldsymbol{x}_{\mathrm{ode}}) - \mathbb{E}_{p_{r}} f_\epsilon}_{\text{predictor error}}
    \;+\; \underbrace{\mathbb{E}_{p_{r}} f_\epsilon - \mathbb{E}_{p_{r}} f}_{\text{corrector bias}}.
\end{equation}
For the predictor error, let $\Phi := \Phi_{t \to r}$ denote the exact flow map, which satisfies $\Phi \# p_t = p_{r}$.
Since the trajectory $u \mapsto \boldsymbol{x}(u)$ has acceleration exactly $\dot{\boldsymbol{v}}$, a Taylor expansion gives $\Phi(\boldsymbol{x}) = \boldsymbol{x} - h\,\boldsymbol{v}_t(\boldsymbol{x}) + \frac{h^2}{2}\,\dot{\boldsymbol{v}}(\boldsymbol{x}) + O(h^3)$. Hence $\boldsymbol{x}_{\mathrm{ode}} - \Phi(\boldsymbol{x}) = -\frac{h^2}{2}\,\dot{\boldsymbol{v}}(\boldsymbol{x}) + O(h^3)$, and for any test function $\psi$,
\begin{equation*}
    \mathbb{E}_{p_t}\, \psi(\boldsymbol{x}_{\mathrm{ode}}) - \mathbb{E}_{p_{r}}\, \psi
    = \mathbb{E}_{p_t}\big[\psi(\boldsymbol{x}_{\mathrm{ode}}) - \psi(\Phi(\boldsymbol{x}))\big]
    = -\frac{h^2}{2}\, \mathbb{E}_{p_t} \langle \dot{\boldsymbol{v}},\, \nabla \psi \rangle + O(h^3).
\end{equation*}
We apply this with $\psi = f_\epsilon$; Taylor-expanding in $\epsilon$ gives $\nabla f_\epsilon = \nabla f + O(\epsilon)$, and since $\epsilon \leq \Lambda h$, the substitution costs only $O(h^3)$.

The corrector bias is that of one Langevin step started from its own target $p_{r}$. Write $\boldsymbol{s} := \boldsymbol{s}_r$ and let $L g := \langle \boldsymbol{s}, \nabla g \rangle + \Delta g$ be the Langevin generator; Due to integration by parts, and recalling $ \boldsymbol{s} = \nabla \log p_r $, we get invariance of $p_r$, which means $\mathbb{E}_{p_{r}} L g = 0$ for all $g$. Taylor expanding $f$ in the increment $\epsilon\,\boldsymbol{s} + \sqrt{2\epsilon}\,\xi$ and taking Gaussian moments,
\begin{equation*}
    \mathbb{E}_{p_{r}} f_\epsilon - \mathbb{E}_{p_{r}} f
    = \epsilon\, \mathbb{E}_{p_{r}}\big[L f\big]
    + \epsilon^2\, \mathbb{E}_{p_{r}}\Big[\tfrac{1}{2}\boldsymbol{s}^\top \nabla^2 f\, \boldsymbol{s} + \langle \boldsymbol{s}, \nabla \Delta f\rangle + \tfrac{1}{2}\Delta^2 f\Big] + O(\epsilon^3).
\end{equation*}
The order-$\epsilon$ term vanishes by invariance. 
For the order-$\epsilon^2$ term, call the bracket $B$. By invariance with $g=Lf$, yielding $\mathbb{E}_{p_r}[\tfrac12 L^2 f]=0$, we can subtract it for free. Expanding $\tfrac12 L^2 f$ and cancelling the common terms leaves only
\begin{equation*}
    B-\tfrac12 L^2 f
    = -\big\langle \nabla\boldsymbol{s},\,\nabla^2 f\big\rangle
      -\tfrac12\big\langle (\boldsymbol{s}\!\cdot\!\nabla)\boldsymbol{s}
      +\Delta\boldsymbol{s},\,\nabla f\big\rangle,
\end{equation*}
i.e.\ $B$ differs from $\tfrac12 L^2 f$ by the Hessian term
$\langle\nabla\boldsymbol{s},\nabla^2 f\rangle$ plus a first-order piece. One integration by parts, using $\nabla p_r=p_r\,\boldsymbol{s}$, removes the
Hessian:
\begin{equation*}
    \mathbb{E}_{p_r}\big\langle \nabla\boldsymbol{s},\nabla^2 f\big\rangle
    = -\,\mathbb{E}_{p_r}\big\langle (\boldsymbol{s}\!\cdot\!\nabla)\boldsymbol{s}
      +\Delta\boldsymbol{s},\,\nabla f\big\rangle,
\end{equation*}
so the two first-order pieces combine into
$\tfrac12\,\mathbb{E}_{p_r}\langle(\boldsymbol{s}\!\cdot\!\nabla)\boldsymbol{s}
+\Delta\boldsymbol{s},\nabla f\rangle$. Because $\boldsymbol{s}=\nabla\log p_r$
is a gradient, $(\boldsymbol{s}\!\cdot\!\nabla)\boldsymbol{s}
=\nabla(\tfrac12\|\boldsymbol{s}\|^2)$ and
$\Delta\boldsymbol{s}=\nabla\Delta\log p_r$, so this field is itself a gradient.
Hence
\begin{equation*}
    \mathbb{E}_{p_{r}} f_\epsilon - \mathbb{E}_{p_{r}} f
    = \frac{\epsilon^2}{2}\, \mathbb{E}_{p_{r}} \Big\langle \nabla\Big(\tfrac{1}{2}\nbr{\boldsymbol{s}}^2 + \Delta \log p_{r}\Big),\, \nabla f \Big\rangle + O(\epsilon^3)
    = \frac{\epsilon^2}{2}\, \mathbb{E}_{p_t} \langle \boldsymbol{G},\, \nabla f \rangle + O(h^3),
\end{equation*}
where moving from time $r$ to time $t$ costs $O(\epsilon^2 h) = O(h^3)$. Combining the two parts,
\begin{equation}\label{eq:lcdef}
    \delta_{\mathrm{lc}}(f)
    = -\frac{h^2}{2}\,\mathbb{E}_{p_t} \Big\langle \dot{\boldsymbol{v}} - \rbr{\tfrac{\epsilon}{h}}^2\boldsymbol{G},\; \nabla f \Big\rangle + O(h^3).
\end{equation}

Next, we show the SDE step adds exactly one term. Let $\tilde{\boldsymbol{x}}$ be the Langevin step run with $\epsilon = \sigma^2 h/2$ and the same noise $\xi$. Then $\sqrt{2\epsilon} = \sigma\sqrt{h}$, so the two schemes share their noise and differ only in where the score is evaluated:
\begin{equation}\label{eq:diff_sde_lc}
    \boldsymbol{x}_{\mathrm{sde}} - \tilde{\boldsymbol{x}}
    = \frac{\sigma^2 h}{2}\,\big[\boldsymbol{s}_t(\boldsymbol{x}) - \boldsymbol{s}_{r}(\boldsymbol{x}_{\mathrm{ode}})\big]
    = \frac{\sigma^2 h^2}{2}\,\dot{\boldsymbol{s}}(\boldsymbol{x}) + O(h^3),
\end{equation}
by first-order Taylor expansion of $(u, \boldsymbol{y}) \mapsto \boldsymbol{s}_u(\boldsymbol{y})$ at $(t, \boldsymbol{x})$ in the direction $\rbr{-h,\, -h\,\boldsymbol{v}_t(\boldsymbol{x})}$, which produces exactly $\dot{\boldsymbol{s}}$. This difference Eq.~\ref{eq:diff_sde_lc} is a function of $\boldsymbol{x}$ alone, so expanding $f$ between the two points, the cross terms with the noise vanish in expectation ($\xi$ is centered and independent of $\boldsymbol{x}$), and
\begin{equation}\label{eq:sdedef}
    \delta_{\mathrm{sde}}(f)
    = \delta_{\mathrm{lc}}(f)\big|_{\epsilon = \sigma^2 h/2} + \frac{\sigma^2 h^2}{2}\,\mathbb{E}_{p_t}\langle \dot{\boldsymbol{s}}, \nabla f\rangle
    = -\frac{h^2}{2}\,\mathbb{E}_{p_t} \Big\langle \dot{\boldsymbol{v}} - \sigma^2 \dot{\boldsymbol{s}} - \tfrac{\sigma^4}{4}\boldsymbol{G},\; \nabla f \Big\rangle + O(h^3).
\end{equation}

Finally, applying Lemma~\ref{lem:weak2W2} to Eq.~\ref{eq:lcdef} with
$\boldsymbol{A}=\dot{\boldsymbol{v}}-(\tfrac{\epsilon}{h})^2\boldsymbol{G}$, and to Eq.~\ref{eq:sdedef} with
$\boldsymbol{A}=\dot{\boldsymbol{v}}-\sigma^2\dot{\boldsymbol{s}}-\tfrac{\sigma^4}{4}\boldsymbol{G}$.
In both cases, $\boldsymbol{A}$ is a gradient field: $\boldsymbol{v}_t$ and $\boldsymbol{s}_t=\nabla\log p_t$ are gradients (hence so are $\dot{\boldsymbol{v}}$ and $\dot{\boldsymbol{s}}$), and $\boldsymbol{G}$ is a gradient by construction. This yields Eq.~\ref{eq:errors}.

For the matched step $\epsilon = \sigma^2 h/2$ we have $(\epsilon/h)^2 = \sigma^4/4$, so by Eq.~\ref{eq:errors} the two leading coefficients are $\frac{h^2}{2}\nbr{\boldsymbol{E}}_p$ and $\frac{h^2}{2}\nbr{\boldsymbol{E} - \sigma^2 \dot{\boldsymbol{s}}}_p$ with $\boldsymbol{E} := \dot{\boldsymbol{v}} - \tfrac{\sigma^4}{4}\boldsymbol{G}$. Both norms are nonnegative, so comparing them is equivalent to comparing their squares; expanding,
\begin{equation*}
    \nbr{\boldsymbol{E}}_p^2 \;<\; \nbr{\boldsymbol{E} - \sigma^2 \dot{\boldsymbol{s}}}_p^2
    = \nbr{\boldsymbol{E}}_p^2 - 2\sigma^2 \big\langle \boldsymbol{E},\, \dot{\boldsymbol{s}} \big\rangle_p + \sigma^4 \nbr{\dot{\boldsymbol{s}}}_p^2
    \quad \Longleftrightarrow \quad
    \big\langle \boldsymbol{E},\, \dot{\boldsymbol{s}} \big\rangle_p \;<\; \frac{\sigma^2}{2}\, \nbr{\dot{\boldsymbol{s}}}_p^{2}.
\end{equation*}
Hence the assumed condition makes the leading coefficient strictly of Langevin step smaller, and since both errors enter at order $h^2$, $\W\big(\law(\boldsymbol{x}_{\mathrm{lc}}),\, p_{r}\big) < \W\big(\law(\boldsymbol{x}_{\mathrm{sde}}),\, p_{r}\big)$ for all small $h$.
\end{proof}

\subsection{Auxiliary Lemmas}

\begin{assumption}\label{ass:lip}
For all $u \in [r, t]$, the velocity $\boldsymbol{v}_u$ is $L$-Lipschitz.
\end{assumption}

\begin{assumption}\label{ass:acc}
For all $u \in [r, t]$ and $\boldsymbol{x} \in \mathbb{R}^d$,
$\nbr{\partial_u \boldsymbol{v}_u(\boldsymbol{x}) + \nabla \boldsymbol{v}_u(\boldsymbol{x})\,\boldsymbol{v}_u(\boldsymbol{x})} \le K$;
that is, trajectories of Eq.~\ref{eq:pf-ode} have acceleration bounded by $K$.
\end{assumption}

\begin{lemma}[One-Step Euler Error for ODEs]\label{lem:ode_predictor}
Under Assumptions~\ref{ass:lip}--\ref{ass:acc}, let $ \boldsymbol{x} \sim q_t$ and define the backward Euler step $\hat{\boldsymbol{x}} = \boldsymbol{x} - h\boldsymbol{v}_t \rbr{\boldsymbol{x}}$, it holds
\begin{equation}
  \W\big(\law(\hat{\boldsymbol{x}}),\, p_{t-h}\big)
  \;\le\; e^{Lh}\,\W(q_t, p_t) + \frac{K}{2}h^2 .
\end{equation}
\end{lemma}
\begin{proof}
Let $\Phi := \Phi_{t\to t-h}$ denote the exact flow map of $\mathrm{d} \boldsymbol{x}_u = \boldsymbol{v}_u \rbr{\boldsymbol{x}_u}\mathrm{d}u$ from time $t$ to time $t-h$. 

We first show that $\Phi$ is $e^{Lh}$-Lipschitz. Let $\boldsymbol{x}_u$ and $\tilde{\boldsymbol{x}}_u$ be two trajectories of the ODE, and define $g(u) := \nbr{\boldsymbol{x}_u - \tilde{\boldsymbol{x}}_u}^2$. Then,
\begin{equation*}
\begin{aligned}
     g^\prime(u) &= 2 \langle \boldsymbol{x}_u - \tilde{\boldsymbol{x}}_u,\, \boldsymbol{v}_u(\boldsymbol{x}_u) - \boldsymbol{v}_u(\tilde{\boldsymbol{x}}_u)\rangle \geq -2 \nbr{\boldsymbol{x}_u - \tilde{\boldsymbol{x}}_u} \cdot \nbr{\boldsymbol{v}_u(\boldsymbol{x}_u) - \boldsymbol{v}_u(\tilde{\boldsymbol{x}}_u)} \\
     & \geq -2 L \nbr{\boldsymbol{x}_u - \tilde{\boldsymbol{x}}_u}^2 = -2Lg(u),
\end{aligned}
\end{equation*}
where Cauchy--Schwarz inequality and the $L$-Lipschitz of velocity $\boldsymbol{v}_u$ from Assumption~\ref{ass:lip}. Hence
$\frac{\mathrm{d}}{\mathrm{d}u}\big(e^{2Lu}g(u)\big) \geq 0$. Therefore, we have
$e^{2L(t-h)}g(t-h) \leq e^{2Lt} g(t)$.
Equivalently,
$$
\nbr{\Phi(\boldsymbol{x}_t)-\Phi(\tilde{\boldsymbol{x}}_t)}
\leq e^{Lh}\nbr{\boldsymbol{x}_t-\tilde{\boldsymbol{x}}_t}.
$$

Next, we compare the exact flow with the Euler step map. By Assumption~\ref{ass:acc}, the map $u \mapsto \boldsymbol{v}_u(\boldsymbol{x}_u)$ is $C^1$
with derivative $\partial_u \boldsymbol{v}_u(\boldsymbol{x}_u) + \nabla \boldsymbol{v}_u(\boldsymbol{x}_u)\, \boldsymbol{v}_u(\boldsymbol{x}_u)$, bounded by $K$. Hence $\nbr{\boldsymbol{v}_u(\boldsymbol{x}_u) - \boldsymbol{v}_t(\boldsymbol{x}_t)} \le K(t-u)$ for
$u \in [t-h,t]$, and since $\Phi(\boldsymbol{x}) = \boldsymbol{x}_t - \int_{t-h}^t \boldsymbol{v}_u(\boldsymbol{x}_u)\,\mathrm{d}u$,
\begin{equation*}
    \nbr{\Phi(\boldsymbol{x}) - \hat{\boldsymbol{x}}}
  = \nbr{\int_{t-h}^t \big(\boldsymbol{v}_u(\boldsymbol{x}_u) - \boldsymbol{v}_t(\boldsymbol{x}_t)\big)\,\mathrm{d}u}
  \le \int_{t-h}^t K(t-u)\,\mathrm{d}u
  = \tfrac{K}{2} h^2 .
\end{equation*}

Let $\boldsymbol{y} \sim p_t$, thus we have $\Phi(\boldsymbol{y}) \sim p_{t-h}$. By the triangle inequality for Wasserstein distance, it holds
\begin{equation*}
  W_2(\mathrm{law}(\hat{\boldsymbol{x}}), p_{t-h})
  \le \nbr{\hat{\boldsymbol{x}} - \Phi(\boldsymbol{y})}_{L^2}
  \le \nbr{\hat{\boldsymbol{x}} - \Phi(\boldsymbol{x})}_{L^2}
    + \nbr{\Phi(\boldsymbol{x}) - \Phi(\boldsymbol{y})}_{L^2}
  \le \tfrac{K}{2}h^2 + e^{Lh}\, W_2(q_t, p_t).
\end{equation*}
\end{proof}

\begin{lemma}[Stationary Score Moment]\label{sub:moment}
Under Assumption~\ref{ass:slc}, it holds
\begin{equation*}
  \mathbb{E}_{p_r}\big[\nbr{\nabla f}^2\big] \;\le\; L d .
\end{equation*}
\end{lemma}
\begin{proof}
Since $\nabla p_r = p_r\,\nabla f$, integration by parts gives
\begin{equation*}
  \int \nbr{\nabla f}^2\, p_r
  = \int \langle \nabla f,\, \nabla p_r \rangle
  = -\int (\Delta f)\, p_r
  = \mathbb{E}_{p_r}\big[\operatorname{tr}(-\nabla^2 f)\big]
  \le L d ,
\end{equation*}
where the last inequality uses $-\nabla^2 f \preceq L I_d$ from Assumption~\ref{ass:slc}, and the boundary terms vanish because $p_r$ has sub-Gaussian tails under strong log-concavity while $\nabla f$ grows at most linearly.
\end{proof}

\begin{lemma}\label{lem:weak2W2}
Let $p,(\mu_h)_{h>0}\in\mathcal{P}_2(\mathbb{R}^d)$, and suppose $\mu_h\to p$ in
$\W$ as $h\to0$. Suppose $\boldsymbol{A}=\nabla\phi\in L^2(p)$ is a gradient field satisfying
\[
  \int f\,d\mu_h-\int f\,dp
  = -\frac{h^2}{2}\int \langle \boldsymbol{A},\nabla f\rangle\,p\,dx + o(h^2),
\]
for all smooth test functions $f$. Then $\W(\mu_h,p)=\frac{h^2}{2}\,\|\boldsymbol{A}\|_{L^2(p)}+o(h^2)$.
\end{lemma}

\begin{proof}
Write $\|\cdot\|:=\|\cdot\|_{L^2(p)}$ and $T:=\mathrm{Id}-\tfrac{h^2}{2}\boldsymbol{A}$.

\emph{Upper bound.} Coupling $p$ with $T_\#p$ through $x\mapsto T(x)$,
$W_2(\mu_h,p)\le W_2(\mu_h,T_\#p)+W_2(T_\#p,p)\le o(h^2)+\tfrac{h^2}{2}\|\boldsymbol{A}\|$,
where the first term is $o(h^2)$ because $\mu_h$ and $T_\#p$ share the same weak
expansion, and the second uses $\int\|x-T(x)\|^2 p=\tfrac{h^4}{4}\|\boldsymbol{A}\|^2$.

\emph{Lower bound.} Testing the hypothesis with $f=-\phi$ gives
\[
  \int(-\phi)\,d\mu_h-\int(-\phi)\,dp
  =\frac{h^2}{2}\int\|\boldsymbol{A}\|^2\,p\,dx+o(h^2)
  =\frac{h^2}{2}\|\boldsymbol{A}\|^2+o(h^2).
\]
On the other hand, integrating $\nabla\phi$ along the 
$W_2$-geodesic from $p$ to $\mu_h$ and applying the first-order (Otto) calculus
on $(\mathcal{P}_2,W_2)$ \citep{otto2001geometry,ambrosio2005gradient},
\[
  \int(-\phi)\,d\mu_h-\int(-\phi)\,dp\le W_2(\mu_h,p)\,\big(\|\boldsymbol{A}\|+o(1)\big),
\]
where the $o(1)$ holds because $\mu_h\to p$ in $W_2$, so the geodesic contracts
to $p$. Combining the two displays,
\[
  \frac{h^2}{2}\|\boldsymbol{A}\|^2+o(h^2)\le W_2(\mu_h,p)\,\big(\|\boldsymbol{A}\|+o(1)\big),
\]
whence $W_2(\mu_h,p)\ge\tfrac{h^2}{2}\|\boldsymbol{A}\|+o(h^2)$. Together with the
upper bound, this proves the claim.

\end{proof}

\section{Experimental Setup Details}\label{app:exp_setup}

Except for the noise level $\eta$ and KL coefficient $\beta$, all other hyperparameters are kept the same across methods. The image-generation configuration largely follows Flow-GRPO.

For SD3.5-M-based image generation, we use 48 prompts per epoch, a group size of 24 generations per prompt, LoRA with rank $r=32$ and $\alpha=64$, a learning rate of $3\times10^{-4}$, and a CFG scale of 4.5. For multi-reward optimization, each reward is normalized separately before being combined. We use weights of 0.4 for HPS-v2.1 and 0.6 for CLIP score. The noise level is set to $\eta=0.7$ for Flow-GRPO, $\eta=0.9$ for CPS, and $\eta=0.8$ for \ours in the OCR setting and $\eta=0.7$ in all other human-preference settings.  The KL coefficient $\beta$ is set to 0 for all human-preference reward experiments. For verifiable rewards such as OCR, we use $\beta>0$ to preserve image quality: $\beta=0.04$ for Flow-GRPO, $\beta=10^{-4}$ for CPS, and $\beta=4\times10^{-3}$ for \ours.

For FLUX.1-Dev-based image generation, we use the same prompt and group sizes as above, set $\beta=0$, and use LoRA with rank $r=64$ and $\alpha=128$ and a learning rate of $3\times10^{-4}$. For multi-reward optimization, rewards are normalized separately and combined with weights 0.7 for HPS-v2.1 and 0.3 for CLIP score. We set noise level $\eta=0.9$ for Flow-GRPO and CPS; for \ours, $\eta=0.8$ for HPS and $\eta=0.9$ for the combined HPS–CLIP setting.

For video generation, we use 16 prompts per epoch, a group size of 8 generations per prompt, a resolution of $480\times480$, 53 frames at 8 fps, and a learning rate of $10^{-5}$. The noise level is set to $\eta=0.25$ for DanceGRPO, $\eta=0.8$ for CPS, and $\eta=0.6$ for \ours.

\section{Additional Results}\label{app:additional_res}

 \vspace{20pt}
\begin{figure}[htbp]
\centering	\includegraphics[width=1.0\linewidth]{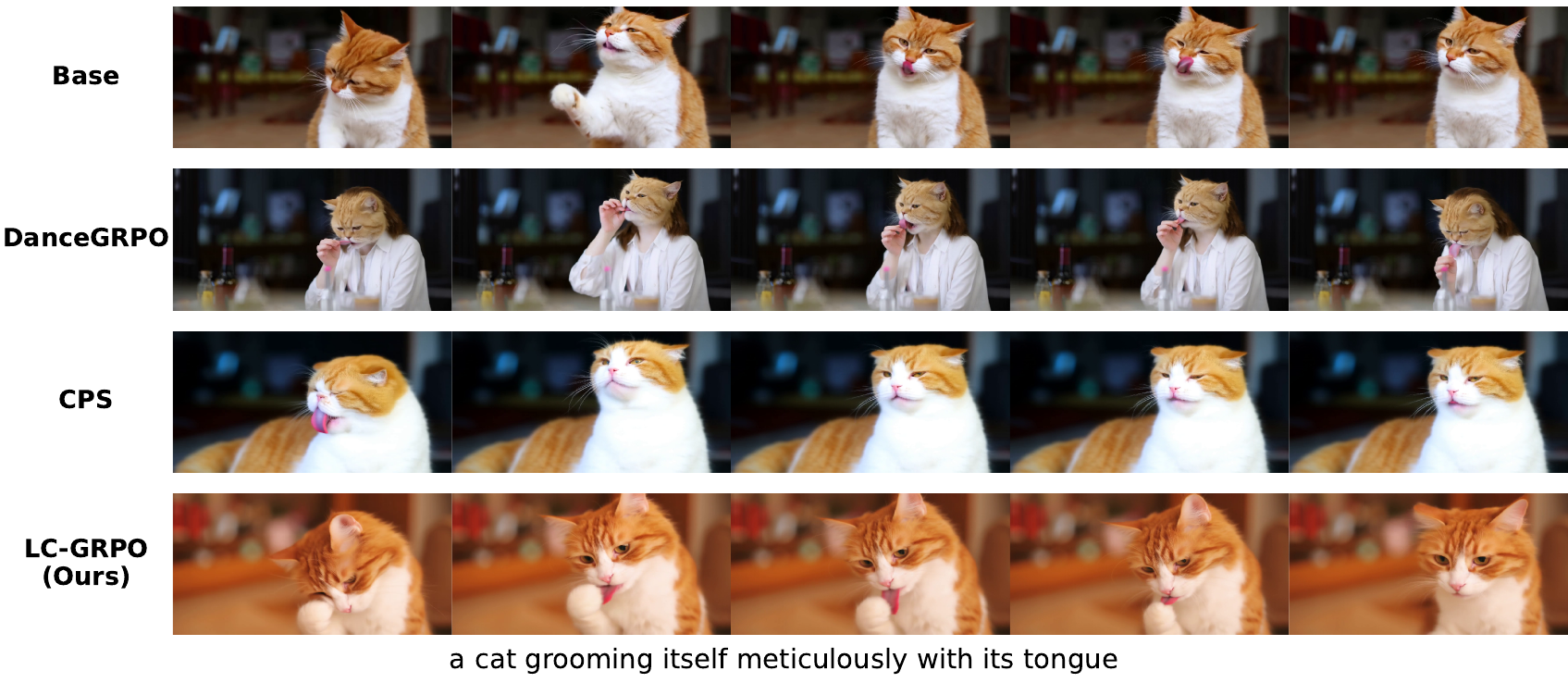}

\includegraphics[width=1.0\linewidth]{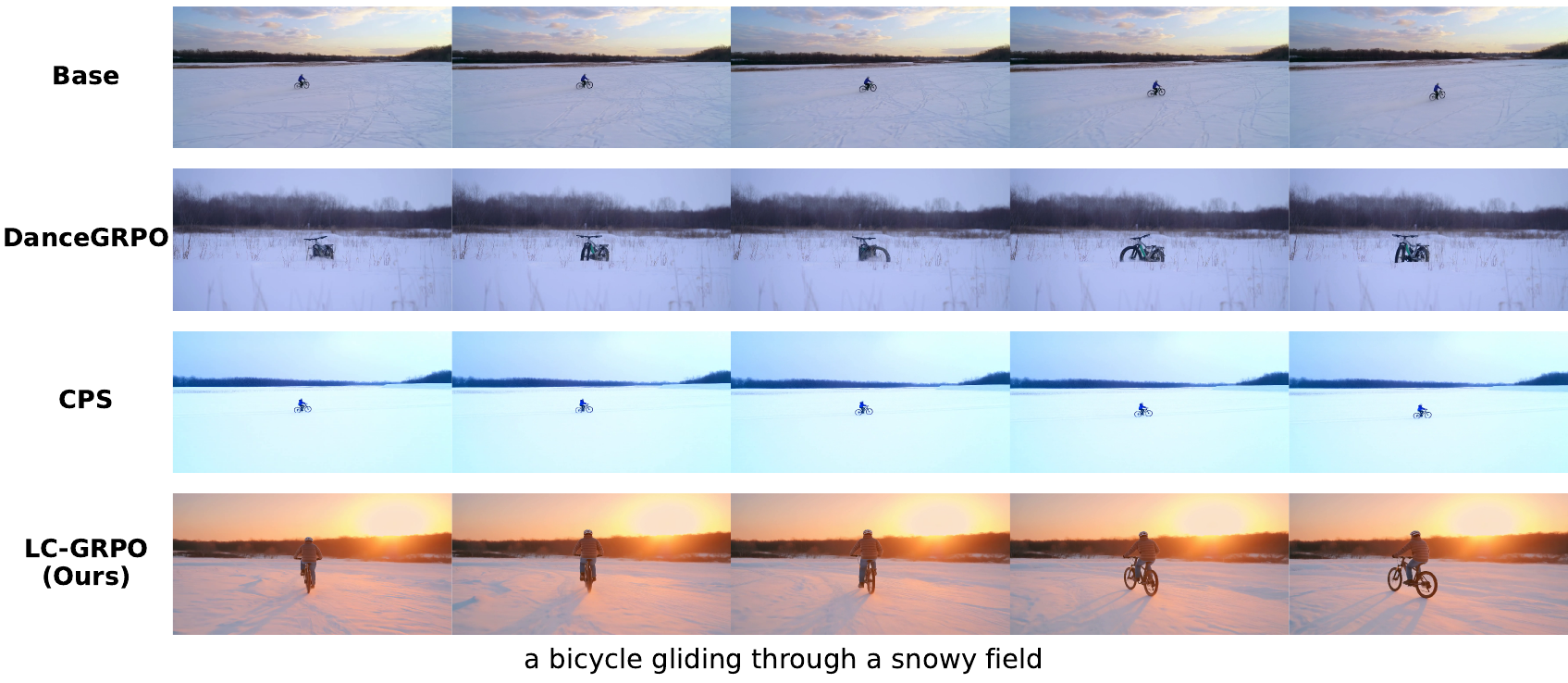} 
   \vspace{-.05in}
	\caption{ Qualitative comparison based on HunyuanVideo trained with the visual quality metric in VideoAlign. }
	\vspace{-.10in}
\end{figure}

\begin{figure}[ht]
\vspace{-11.2pt}
\centering	\includegraphics[width=1.0\linewidth]{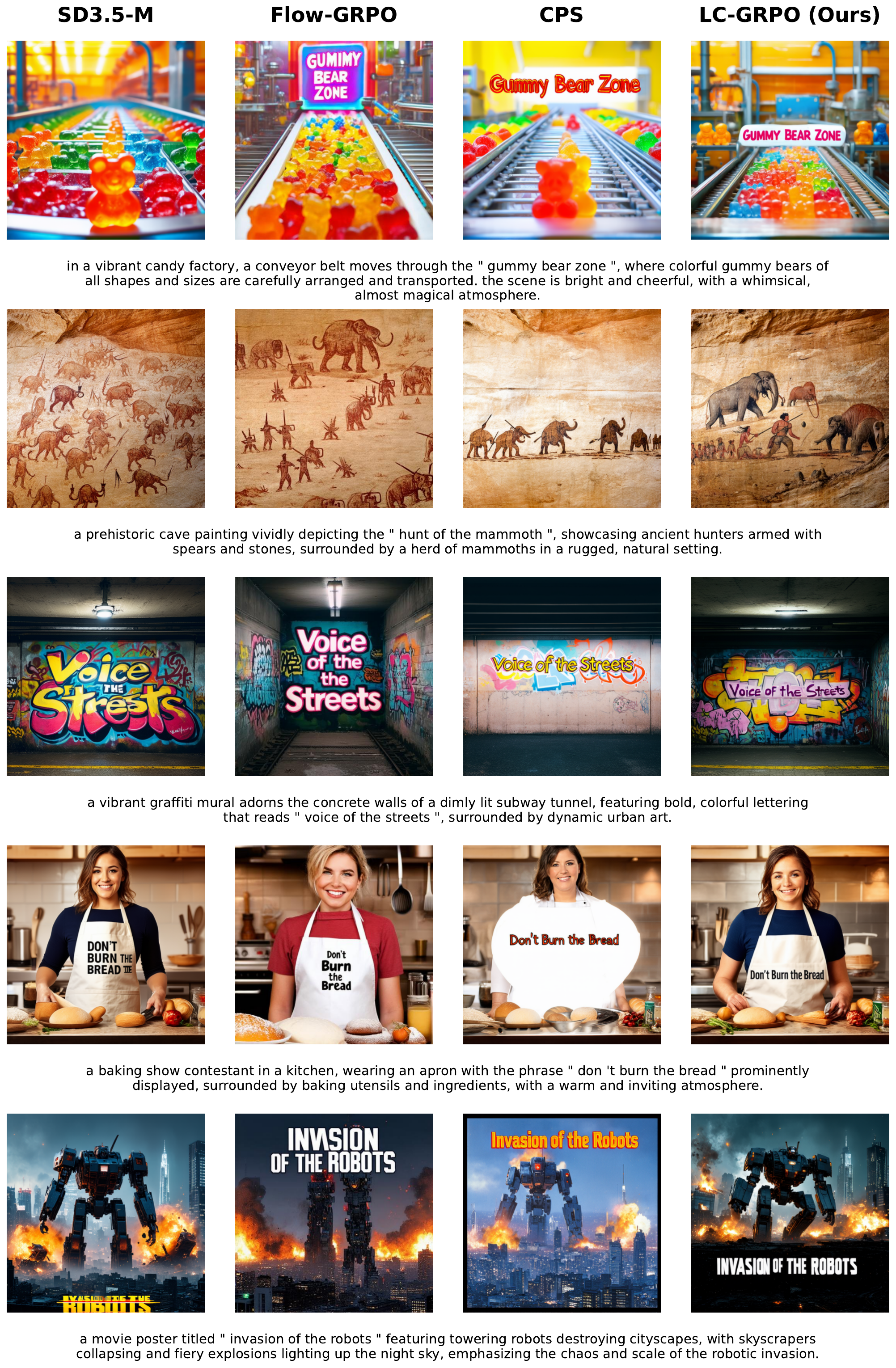}\\
   \vspace{-.05in}
	\caption{ Additional qualitative comparison based on SD3.5-M trained with OCR reward. }
	\vspace{-.10in}
\end{figure}

\begin{figure}[ht]
\centering	\includegraphics[width=1.0\linewidth]{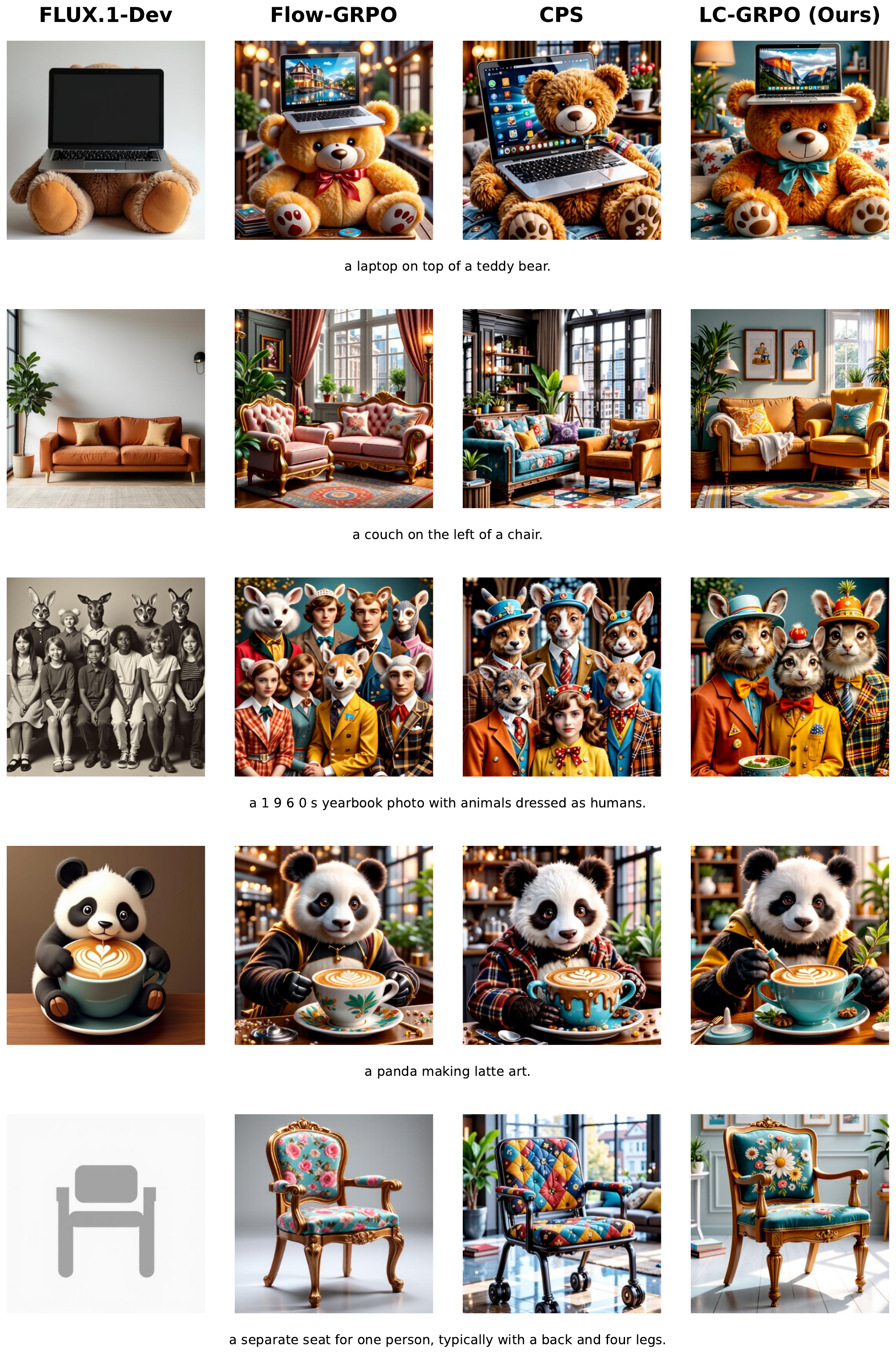}\\
   \vspace{-.05in}
	\caption{ Additional qualitative comparison based on FLUX.1-Dev trained with HPS-v2.1 reward, testing on DrawBench. }
	\vspace{-.10in}
\end{figure}

\end{document}